\documentclass{article}

\usepackage[preprint]{neurips_2025}
\usepackage{algorithm}
\usepackage{algorithmic}
\usepackage{amssymb}
\usepackage{amsfonts}
\usepackage{amsmath}
\usepackage{amsthm}
\usepackage{graphicx} 

\usepackage{amsthm, amsmath, amssymb, bbm, mathabx}
\allowdisplaybreaks 
\usepackage{graphicx}
\usepackage{enumerate}
\usepackage{pifont}
\usepackage{booktabs}
\usepackage{multirow}
\usepackage{multicol}
\usepackage{stfloats}
\usepackage{xspace}
\usepackage{thmtools}
\usepackage{thm-restate}
\usepackage{hyperref}
\usepackage{cleveref}
\usepackage{anyfontsize} 
\usepackage{makecell}
\usepackage{hhline}
\usepackage{colortbl}
\usepackage{xpatch}
\usepackage{scalerel,stackengine}
\usepackage{sidecap}
\usepackage[normalem]{ulem}
\usepackage{tcolorbox}

\makeatletter
\def\@fnsymbol#1{\ensuremath{\ifcase#1\or *\or \dagger\or \ddagger\or
  \mathsection\or \mathparagraph\or \|\or \diamond \or **\or \dagger\dagger
  \or \ddagger\ddagger \else\@ctrerr\fi}}
\makeatother

\makeatletter
\newcommand{\printfnsymbol}[1]{%
  \textsuperscript{\@fnsymbol{#1}}%
}
\makeatother

\newtheorem{theorem}{Theorem}
\newtheorem{lemma}{Lemma}
\newtheorem{corollary}[theorem]{Corollary}

\newtheorem{definition}{Definition}

\crefname{condition}{Condition}{Conditions}

\newtheorem{assumption}[theorem]{Assumption}
\crefname{assumption}{Assumption}{Assumptions}

\theoremstyle{definition}
\newtheorem{remark}{Remark}

\newcommand{\abs}[1]{\left| #1 \right|}

\newcommand\mygei{\mathrel{\stackrel{\makebox[0pt]{\mbox{\normalfont\tiny (i)}}}{\ge}}}

\newcommand{\wt}[1]{\widetilde{#1}}

\newcommand{\cA}{\mathcal{A}}

\newcommand{\cE}{\mathcal{E}}
\newcommand{\cF}{\mathcal{F}}

\newcommand{\cI}{\mathcal{I}}

\newcommand{\cM}{\mathcal{M}}

\newcommand{\cS}{\mathcal{S}}

\newcommand{\cX}{\mathcal{X}}

\newcommand{\cZ}{\mathcal{Z}}

\newcommand{\bbE}{\mathbb{E}}

\newcommand{\bbP}{\mathbb{P}}
\newcommand{\bbQ}{\mathbb{Q}}

\newcommand{\ee}{\textup{e}}

\newcommand{\kl}{\mathsf{kl}}

\newcommand{\poly}{\mathsf{poly}}

\newcommand{\Roma}[1]{\uppercase\expandafter{\romannumeral#1}}

\newcommand{\SA}{\cS \times \cA}

\newcommand{\red}[1]{\textcolor{red!80}{#1}}
\newcommand{\blue}[1]{\textcolor{blue!80}{#1}}

\usepackage[colorinlistoftodos, textwidth=18mm]{todonotes}

\def\shownotes{1}
\ifnum\shownotes=0
\newcommand{\todorz}[1]{}
\newcommand{\todorzout}[1]{}
\newcommand{\todossdout}[1]{}
\newcommand{\todossd}[1]{}
\newcommand{\todoruizhe}[1]{}
\else
\newcommand{\todorz}[1]{\todo[color=blue!10, inline]{\small RZ: #1}}
\newcommand{\todorzout}[1]{\todo[color=blue!10]{\scriptsize RZ: #1}}
\newcommand{\todossdout}[1]{\todo[color=red!10]{\scriptsize SSD: #1}}
\newcommand{\todossd}[1]{\todo[color=red!10, inline]{\small SSD: #1}}

\newcommand{\todoruizhe}[1]{\todo[color=purple!10, inline]{\small Ruizhe: #1}}

\fi

\title{Sharp Gap-Dependent Variance-Aware Regret Bounds for Tabular MDPs}
\author{
    Shulun Chen\thanks{Work done while Shulun Chen was visiting the University of Washington.}\\
    Tsinghua University\\
    \texttt{chensl22@mails.tsinghua.edu.cn}
    \AND
    Runlong Zhou\\
    University of Washington\\
    \texttt{vectorzh@cs.washington.edu}\\
    \And
    Zihan Zhang\\
    University of Washington\\
    \texttt{zihanz46@uw.edu}\\
    \AND
    \hspace{2em}Maryam Fazel\\
    \hspace{2em}University of Washington\\
    \hspace{2em}\texttt{mfazel@uw.edu}\\
    \And
    \hspace{1em}Simon S. Du\\
    \hspace{1em}University of Washington\\
    \hspace{1em}\texttt{ssdu@cs.washington.edu}\\
}

\newcommand{\State}{\mathcal{S}}
\newcommand{\Action}{\mathcal{A}}
\newcommand{\gapmin}{\Delta_{\mathrm{min}}}

\newcommand{\Var}{\mathtt{Var}\xspace}
\newcommand{\Varmax}{\Var_{\max}}
\newcommand{\cVarmax}{\Var_{\max}^{\textup{c}}}
\newcommand{\Regret}{\mathrm{Regret}}

\newcommand{\Zopt}{\mathcal{Z}_{\mathrm{opt}}}
\newcommand{\Zsub}{\mathcal{Z}_{\mathrm{sub}}}
\newcommand{\clip}[2]{\operatorname{clip}\left[#1|#2\right]}

\begin{document}

\maketitle

\begin{abstract}
We consider the gap-dependent regret bounds for episodic MDPs.
We show that the Monotonic Value Propagation (MVP) algorithm (\cite{zhang2024settling}) achieves a variance-aware gap-dependent regret bound of 
$$\tilde{O}\left(\left(\sum_{\Delta_h(s,a)>0} \frac{H^2 \log K \land \cVarmax}{\Delta_h(s,a)} +\sum_{\Delta_h(s,a)=0}\frac{ H^2 \land \cVarmax}{\Delta_{\mathrm{min}}} + SAH^4 (S \lor H) \right) \log K\right),$$ 
where $H$ is the planning horizon, $S$ is the number of states, $A$ is the number of actions, and $K$ is the number of episodes. Here, $\Delta_h(s,a) =V_h^* (a) - Q_h^* (s, a)$ represents the suboptimality gap and $\Delta_{\mathrm{min}} := \min_{\Delta_h (s,a) > 0} \Delta_h(s,a)$.
The term $\cVarmax$ denotes the maximum conditional total variance, calculated as the maximum over all $(\pi, h, s)$ tuples of the expected total variance under policy $\pi$ conditioned on trajectories visiting state $s$ at step $h$. 
$\cVarmax$ characterizes the maximum randomness encountered when learning any $(h, s)$ pair.
Our result stems from a novel analysis of the weighted sum of the suboptimality gap and can be potentially adapted for other algorithms.
To complement the study, we establish a lower bound of 
$$\Omega \left( \sum_{\Delta_h(s,a)>0} \frac{H^2 \land \cVarmax}{\Delta_h(s,a)}\cdot \log K\right),$$ 
demonstrating the necessity of dependence on $\cVarmax$ even when the maximum unconditional total variance (without conditioning on $(h, s)$) approaches zero.
\end{abstract}

\section{Introduction} \label{sec:intro}

Reinforcement learning (RL, \citet{sutton1998reinforcement}) is an interactive decision-making problem where an agent gains information from an unknown environment through taking actions, with the goal of maximizing the total reward.
RL has a wide range of applications, such as robotics and control \citep{lillicrap2015continuous}, games \citep{silver2016mastering}, finance \citep{nevmyvaka2006reinforcement}, healthcare \citep{liu2017deep}, and recommendation systems \citep{chen2019top}.

The most canonical setting in RL is episodic learning in tabular Markov decision processes (MDPs), where the agent interacts with the MDP for $K$ episodes, each episode allowing exactly $H$ steps taken.
Under this setting, we choose \emph{cumulative regret} as the performance criteria, which should scale sublinearly with $K$ to indicate that the agent is making progress by shortening the performance difference between the policy $\pi^k$ played in episode $k$ and the optimal policy $\pi^*$.
Most work \citep{azar2017minimax,jin2018q,dann2019policy,zhang2020almost,zhang2021reinforcement} in this topic focused on \emph{minimax regret} that is the worst-case guarantee for the algorithms over all the MDPs.
Typically, these minimax regret bounds have main order terms scaling with $\sqrt{K}$.

The MDPs in practice often enjoy benign structures, so the above-mentioned algorithms may perform far better than their worst-case guarantees.
Consequently, \emph{problem-dependent} regret bounds are of great interest.
\emph{Variance-dependent} regret bounds \citep{talebi2018variance,zanette2019tighter,zhou2023sharp,zhang2024settling} are informative when the MDP is near-deterministic.
This type of regret bounds have main order terms scaling with $\sqrt{\Var \cdot K}$ where $\Var$ is a symbol for some variance quantity (might be different across different works).
For deterministic MDPs and MDPs such that $V_h^* (s) = V_h^* (s')$ for any $h, s, s'$, $\Var = 0$.

Meanwhile, \emph{gap-dependent} regret bounds \citep{simchowitz2019non,yang2021q,dann2021beyond,xu2021fine,zheng2024gap} are especially favored when for every $h, s$, the optimal value $V_h^* (s)$ is better than other suboptimal values $Q_h^* (s, a)$ by a margin.
Formally, let $\Delta_h (s, a) := V_h^* (s) - Q_h^* (s, a)$ and $\gapmin := \min \{ \Delta_h (s, a) \mid (h, s, a) \in [H] \times \cS \times \cA, \Delta_h (s, a) > 0 \}$, then a typical gap-dependent regret bound is
\begin{align}
    \wt{O} \left( \left( \sum_{(h, s, a) \in \Zsub} \frac{1}{\Delta_h (s, a)} + \frac{\abs{\Zopt}}{\gapmin} + \poly (H, S, A) \right) \blue{\poly (H)} \cdot \red{\log K} \right), \label{eq:typical_gap}
\end{align}
where $\Zsub$ is the set of all suboptimal $(h, s, a)$ tuples, $\Zopt$\footnote{\citet{xu2021fine} used a more fine-grained notion named $\cZ_{\textup{mul}}$ instead.} is the set of all optimal $(h, s, a)$ tuples, and $\wt{O}$ hides $\poly \log (S, A, H, 1 / \gapmin, 1 / \delta)$ terms.
When $K$ is large enough, gap-dependent regrets grow much slower than minimax and variance-dependent (when $\Var > 0$) regrets.

A natural yet fundamental question about problem-dependent regrets is:
\begin{center}
    \textbf{\emph{What is the tightest problem-dependent regret while considering both variance and gap?}}
\end{center}
If such a regret outperforms variance-only-dependent and gap-only-dependent regrets \emph{asymptotically} (as $T \to \infty$) while also being nearly minimax optimal, it is actually \textbf{\emph{best-of-three-worlds}}!

To address the above problem, there are two factors that can be improved in previous gap-dependent regrets.
First is the dependence on variance quantities.
Only \citet{simchowitz2019non,zheng2024gap} contain variance-dependent terms in their gap-dependent regrets, while their variance quantities are defined as the \emph{maximum per-step} variance, $\bbQ^* \le H^2$.
This quantity is first defined in \citet{zanette2019tighter}, and all of them use $H \bbQ^*$ as an \emph{almost-sure upper bound} on variances.
This upper bound can be substantially larger than an \emph{expected total} variance (such as Definitions 5 and 6 in \citet{zhou2023sharp}).
From this side, a tighter dependence on an expected total variance can improve the regret.

Second is the dependence on $H$.
Specifically, when compared under the \emph{time-inhomogeneous} setting, the $\blue{\poly (H)}$ factors in \Cref{eq:typical_gap} are $H^3$, $H^6$, $H^5$, and $H^5$ in \citet{simchowitz2019non,yang2021q,xu2021fine,zheng2024gap}, respectively.
\citet{simchowitz2019non} provides a lower bound of $\Omega \left(\sum_{s, a} H^2 / \Delta_1 (s, a)\right)$, which indicates the chance of shaving out extra $H$ dependence.

\paragraph{Our contributions.}

We analyze the gap-dependent regret of the Monotonic Value Propagation (MVP, \citet{zhang2024settling} version) algorithm, which is a model-based algorithm already proven to be near-optimal in the sense of minimax and variance-only-dependent regrets.
After careful analysis, we show that the gap-dependent regret depends on a variance quantity $\cVarmax \le H\bbQ^*$, and the worst-case dependency on $H$ is $H^2$.
\textbf{\emph{We improve the above-mentioned two factors simultaneously.}}
Formally, with probability at least $1 - \delta$, the regret in $K$ episodes by MVP is bounded as
\begin{align}
   \wt{O} \left( \left(\sum_{(h, s, a) \in \Zsub} \frac{H^2\log K \land \cVarmax}{\Delta_h(s,a)}+\frac{(H^2 \land \cVarmax ) |\Zopt|}{\gapmin}+SAH^4 (S \lor H) \right) \log K \right).\label{eq:ub}
\end{align}
To the best of our knowledge, we are the first to incorporate a \emph{tighter} variance quantity into gap-dependent regrets, and the worst-case dependency of $H^2$ in gap-dependent terms is also the state-of-the-art (see \Cref{tab:comparison}).

To complement our upper bound, we provide a lower bound (see Theorem~\ref{thm:lb}) of 
\begin{align}
\Omega\left(\sum_{(h,s,a)\in \Zsub} \frac{H^2 \land \cVarmax}{\Delta_h(s,a)}\cdot \log K\right).\nonumber
\end{align}

With this lower bound, we show that the first term in the upper bound \eqref{eq:ub} is tight (modulo log terms). This implies that (i) It is necessary to introduce the conditional total variance (see Definition~\ref{def:varmax}) to derive a variance-aware gap-dependent bound. In comparison, the unconditional total variance (see Definition~\ref{def:unvarmax}) is sufficient for variance-aware minimax bounds (e.g., \cite{zhou2023sharp}); (ii) When the first term in \eqref{eq:ub} dominates, the order of $H$ cannot be improved.

\paragraph{Technical novelty.}

We propose a new variance metric to describe the upper bound of regret in gap-dependent MDPs. Our version of variance metric considers the conditional total variance to allow for some states with small visiting probability to accumulate a large regret over the whole training progress.

To derive a tighter regret bound using our new metric, we utilize a novel analysis which reweighs the suboptimality gaps.
Our approach does not require the clipping and recursion method in \cite{simchowitz2019non} for the main bound; instead, we directly prove that a certain weight sum over all suboptimality gaps times the visitation counts is bounded by a lower-order term of visitation counts, and establish a congregated upper bound of all visitation counts. We believe our approach is novel and reveals fundamental facts about suboptimality gaps.

We also propose a more refined version of clipping for optimal actions. Our version of clipping utilizes the new conditional variance metric while also providing an $O(H^2)$ worst case bound for $\Delta_{\mathrm{min}}$-dependent terms.

Finally, we prove that the $\Delta_h (s, a)$ terms in our upper bound match the lower bound modulo $\log$ factors. The construction is based on a reduction to Bernoulli bandits. A key insight is that low-frequency states, though often neglected in deriving minimax regret bounds, can still contribute substantially to regret in gap-dependent bounds.


\paragraph{Paper overview.}
In \Cref{sec:rel}, we introduce previous research about gap-dependent regret bound.
In \Cref{sec:pre}, we list the basic concepts of MDPs and define the conditional variance.
In \Cref{sec:main_results}, we describe the MVP algorithm and provide a proof sketch of the gap-dependent regret upper bound.
We conclude our paper in \Cref{sec:lb} with a matching lower bound.

\section{Related works}\label{sec:rel}

\paragraph{Gap-dependent regrets and sample complexities.}
Research on gap-dependent regrets originates from multi-armed bandits, which are special MDPs with $H = S = 1$.
\citet{auer2002finite} showed a $\sum_{a \in \cZ_{\textup{sub}}} \log K / \Delta (a)$ type regret when running an UCB algorithm on MABs.
\citet{bubeck2012regret} proposed algorithms achieving a $\sum_{a \in \cZ_{\textup{sub}}} (\Delta (a) + \log (1 / \varepsilon) / \Delta (a))$ \emph{bounded} regret given knowledge of the maximum reward $\max_a r (a)$ as well as a lower bound $\epsilon > 0$ of $\Delta$.

Aside from the works studying finite-horizon tabular MDPs mentioned in \Cref{sec:intro}, there is a line of work under the setting of gap-dependent regrets for infinite-horizon tabular MDPs \citep{auer2006logarithmic,tewari2007optimistic,auer2008near,ok2018exploration}, while in these works, the gaps are usually defined as the difference between policies instead of actions.
Recently, gap-dependent regrets have been studied for risk-sensitive RL \citep{fei2022cascaded}, linear/general function classes \citep{he2021logarithmic,papini2021reinforcement,velegkas2022reinforcement}, and Markov games \citep{dou2022gap}.

Gap-dependent sample complexities under online \citep{jonsson2020planning,marjani2020best,al2021navigating,wagenmaker2022beyond,tirinzoni2022near,wagenmaker2022instance,tirinzoni2023optimistic} and offline \citep{wang2022gap,nguyen2023instance} RL setting are also widely studied.

\paragraph{Minimax optimal regrets.}
Under the setting of time-inhomogeneous MDPs, algorithms achieving a high-probability regret upper bound of $\wt{O} (\sqrt{H^3 S A K})$ are \emph{(nearly) minimax optimal}.
There have been many works with this guarantee while optimizing the lower order terms: \citet{azar2017minimax,osband2017posterior,zanette2019tighter,simchowitz2019non,zhang2019regret,zhang2020almost,zhang2021reinforcement,menard2021ucb,li2021breaking,xiong2022near,zhou2023sharp,zhang2024settling}.
Notably, \citet{zhang2024settling} derived the tightest $\wt{O} (\sqrt{H^3 S A K} \land H K)$ regret up to logarithm factors.

\paragraph{Variance-dependent regrets.}
\citet{talebi2018variance} studied variance-dependent regrets for infinite horizon learning under strong assumptions on ergodicity of the MDPs.
\citet{zanette2019tighter} defined and incorporated the maximum per-step conditional variance, $\bbQ^*$, and first proved a $\wt{O} (\sqrt{H \bbQ^* \cdot S A K})$ regret for the finite-horizon setting.
\citet{zhou2023sharp,zhang2024settling} proved regrets depending on expected total variances (see our \Cref{def:varmax} for one of their quantities) that are more fine-grained than the coarse $H \bbQ^*$ upper bound.
Variance-dependent regrets have also been studied for bandits \citep{zhang2021improved,zhou2021nearly,kim2022improved,dai2022variance}.

\paragraph{Other problem-dependent regrets.}
Under infinite-horizon setting, \citet{bartlett2012regal,fruit2018efficient} studied regrets depending on the span of the optimal value function.
There are works studying first-order regrets, whose main order terms depend on value functions: \citet{jin2020reward,wagenmaker2022first,huang2023tackling}.

\colorlet{shadecolor}{gray!40}
\begin{table*}[t] 
    \centering
    \small
    \resizebox{0.99\linewidth}{!}{%
        \renewcommand{\arraystretch}{1.5}
        \begin{tabular}{|c|c|c|c|}
            \hline 
            \textbf{Algorithm} & \textbf{Gap-dependent Regret} & \scriptsize \makecell{\textbf{Variance-} \\ \textbf{dependent}} & \textbf{Minimax Optimal}\\
            
            \hhline{|=|=|=|=|}
            \multirow{2}{*}{\makecell{StrongEuler \\ \citep{simchowitz2019non}}} & $\wt{O} ( ((\sum_{h, s, a} H \bbQ^* / (\Delta_h(s,a) \lor \gapmin) + SAH^4 (S \lor H)) \cdot \log K )$ & Yes ($H \bbQ^*$) & \multirow{2}{*}{\makecell{\textbf{Yes} ($\wt{O} (\sqrt{H^3 S A K})$)}} \\
            
            \cline{2-3}
            & $\wt{O} ( (\sum_{(h, s, a) \in \Zsub} H^3 / \Delta_h(s,a) + |\Zopt| H^3 / \gapmin+SAH^4 (S \lor H) ) \log K )$ & No & \\

            \hline
            \makecell{Q-learning (UCB-H) \\ \citep{yang2021q}} & $\wt{O} (H^6 S A / \gapmin \cdot \log K)$ & No & \makecell{No ($\wt{O} (\sqrt{H^5 S A K})$) \\ \citep{jin2018q}} \\

            \hline
            \makecell{AMB \\ \citep{xu2021fine}} & $\wt{O} ( (\sum_{(h, s, a) \in \Zsub} H^5  / \Delta_h(s,a) + |\cZ_{\textup{mul}}| H^5 / \gapmin ) \log K + S A H^2 )$  & No & Not Provided \\

            \hline
            \makecell{UCB-Advantage \\ \citep{zheng2024gap}} & $\wt{O} ((H \bbQ^* + H) H^2 S A / \gapmin \cdot \log K + S^2 A H^9 \cdot \log^2 K)$  & Yes ($H \bbQ^*$)  & \makecell{\textbf{Yes} ($\wt{O} (\sqrt{H^3 S A K})$) \\ \citep{zhang2020almost}} \\

            \hline
            \makecell{Q-EarlySettled-Advantage \\ \citep{zheng2024gap}} & $\wt{O} ((H \bbQ^* + H^2) H^2 S A / \gapmin \cdot \log K + S A H^7 \cdot \log^2 K) $  & Yes ($H \bbQ^*$) & \makecell{\textbf{Yes} ($\wt{O} (\sqrt{H^3 S A K})$) \\ \citep{li2021breaking}} \\

            \hline
            \rowcolor{shadecolor} \Gape[0pt]
            [2pt]{\makecell{MVP \\ This work}} & $\wt{O} ( (\sum_{(h, s, a) \in \Zsub} (H^2\log K \land \cVarmax) / \Delta_h(s,a) + |\Zopt| (H^2 \land \cVarmax) / \gapmin+SAH^4 (S \lor H) ) \log K )$ & \textbf{Yes} ($H^2 \land \cVarmax$) & \makecell{\textbf{Yes} ($\wt{O} (\sqrt{H^3 S A K}) $) \\ \citep{zhang2024settling}} \\

            \hhline{|=|=|=|=|}
            \rowcolor{shadecolor} \Gape[0pt]
            [2pt]{\makecell{Lower Bound \\ This work}} & $\wt{\Omega} ( (\sum_{(h, s, a) \in \Zsub} (H^2 \land \cVarmax) / \Delta_h(s,a) \cdot \log K )$ & - & - \\
            \hline
        \end{tabular}
    }
    \caption{Comparison between different algorithms and their gap-dependent regrets for \emph{time-inhomogeneous} MDPs.
    The result in \citet{simchowitz2019non} is scaled accordingly as it originally studied \emph{time-homogeneous} MDPs.
    \textbf{Variance-dependence:} whether the gap-dependent regret is also variance-dependent.
    $\cVarmax \le H\bbQ^*$, so dependence on $H^2 \land \cVarmax$ is tighter.
    \textbf{Minimax Optimal:} whether the analyzed algorithm achieves a $\wt{O} (\sqrt{H^3 S A K})$ (main order) minimax regret.
    \citet{xu2021fine} did not provide such a guarantee.}
    \label{tab:comparison}
\end{table*}

\section{Preliminaries}\label{sec:pre}

\paragraph{Notations.}

For any event $\cE$, let $\mathbf 1\{\cE\}$ be the indicator function of $\cE$.
For any set $\cX$, we use $\Delta^\cX$ to denote the probability simplex over $\cX$.
For any positive integer $n$, we denote $[n] := \{1, 2, \ldots, n\}$.
$\tilde{O}, \tilde{\Omega}, \lesssim$ hide $\poly \log (S, A, H, 1 / \gapmin, 1 / \delta)$ factors.

\paragraph{Finite-horizon MDPs and trajectories.}
A finite-horizon MDP is described by a tuple $M = (\cS, \cA, H, P, R, \mu)$.
$\cS$ is the finite state space with size $S$ and $\cA$ is the finite action space with size $A$.
$H$ is the planning horizon.
For any $(s, a, h) \in \cS \times \cA \times [H]$, $P_{s,a,h}\in\Delta^\cS$ is the transition function and $R_{s,a,h}\in\Delta^{[0,H]}$ is the reward distribution with mean $r_h: \SA \to [0, H]$.
$\mu \in \Delta^\cS$ is the initial state distribution.
A trajectory $\{s_1,a_1,r_1',s_2,a_2,r_2',\cdots,s_H,a_H,r_H'\}$ is sampled with $s_1 \sim \mu, s_{h+1}\sim P_{s_h, a_h,h},r_h'\sim R_{s_h,a_h,h}$ where $a_h$ can be chosen arbitrarily.

Unlike most common settings, we relax the standard assumption that $R_{s, a, h} \in \Delta^{[0, 1]}$ (uniformly bounded reward) and instead assume a bounded total reward setting (\Cref{asp:bounded_total_reward}).
Problems under this setting can contain a spike in reward and are therefore harder than standard problems.

\begin{assumption} [Bounded total reward] \label{asp:bounded_total_reward}
 We assume that $\sum_{h=1}^H r_h' \le H$ for any possible trajectory.
\end{assumption}

\paragraph{Policies.}
A history-independent deterministic policy $\pi$ chooses an action based on the current state and time step.
Formally, $\pi = \{ \pi_h \}_{h \in [H]}$ where $\pi_h : \cS \to \cA$ maps a state to an action.
Any trajectory sampled by $\pi$ satisfies $a_h = \pi_h (s_h)$.
For any random variable $X$ related to a trajectory, we denote $\mathbb E^\pi[X]$ and $\mathbb V^\pi[X]$ as the expectation and variance of $X$ when the trajectory is sampled under $\pi$.

\paragraph{Value functions and $Q$-functions.}
Given $\pi$, we define its value function and $Q$-function as
\begin{align*}
    V_h^\pi (s) := \bbE^\pi \left[\left. \sum_{t = h}^H r_t\ \right|\ s_h = s \right], \quad
    Q_h^\pi (s, a) := \bbE^\pi \left[\left. \sum_{t = h}^H r_t\ \right|\ (s_h, a_h) = (s, a) \right].
\end{align*}
It is easy to verify that $Q_h^\pi (s, a) = r_h (s, a) + P_{s, a, h} V_{h + 1}^\pi$.
We define $V_0^\pi := \mathbb E^{s\sim\mu}[V_1^\pi(s)]$ as the expected total reward when executing policy $\pi$.

\paragraph{Learning objective.}
Episodic RL on MDPs proceeds for a total of $K$ episodes.
At the beginning of episode $k$, the learner chooses a policy $\pi^k$ and uses it to sample a trajectory.

We aim to maximize $V_0^\pi$.
Using dynamic programming, we can find a policy $\pi^*$ maximizing all $Q_h^\pi (s, a)$ simultaneously, and we denote $V^* := V^{\pi^*}, Q^* := Q^{\pi^*}$.

Performance is evaluated by the cumulative regret:
\[\Regret (K) := \sum_{k=1}^K \left(V_0^*-V_0^{\pi^k}\right).\]

\paragraph{Gap quantities.}
The suboptimality gap is defined as follows:
\[\Delta_h(s,a) :=V_h^*(s)-Q_h^*(s,a).\]
The sets of optimal and suboptimal actions are defined as \[\Zopt=\{(s,a,h)\in\mathcal S\times\mathcal A\times[H]:\Delta_h(s,a)=0\}, \quad \Zsub = \mathcal S\times\mathcal A\times[H] \backslash \Zopt.\]
The minimum gap $\gapmin=\min_{(s,a,h)\in\Zsub}\Delta_h(s,a)$ is the smallest positive gap.
WLOG, we only consider MDPs with nonempty $\Zsub$.

\paragraph{Variance quantities.}

The variance at each $(s, a, h)$ tuple \citep{zanette2019tighter,simchowitz2019non} is defined as
\[\Var_h^*(s,a) :=\mathbb V^{r\sim R_{s,a,h},s'\sim P_{s,a,h}}[r+V_{h+1}^*(s')].\]
The maximum per-step conditional variance is defined as $\bbQ^* := \max_{h, s, a} \Var_h^*(s,a)$.
Previous works including \citet{zheng2024gap} use $H \bbQ^*$ which could be as large as $H^3$ in their variance-dependent terms.

The maximum \emph{unconditional} total variance has been introduced in prior works \citep{zhou2023sharp,zhang2024settling} when studying variance-dependent regret bounds for MDPs.

\begin{definition}[Maximum unconditional total variance]\label{def:unvarmax}
\begin{align*}
\Varmax := \max_{\pi}\mathbb{E}^{\pi}\left[\sum_{h=1}^H\Var_{h}^*(s_{h},a_{h}) \right].
\end{align*}
\end{definition}

These works showed that $\Varmax \lesssim \min\{H \bbQ^*, H^2\}$ and incorporated it in the main order terms of variance-only-dependent regrets for better results.
However, as we will discuss in \Cref{thm:lb}, variance-aware gap-dependent regrets \emph{must} scale with separate variance quantities for each $(s, h)$ pair, even for those hard to visit.
Thus, the quantity should be conditioned on $(s, h)$.
We propose the following quantity as the maximum \emph{conditional} total variance:
    
\begin{definition}[Maximum conditional total variance]\label{def:varmax}
\begin{align*}
    \cVarmax: =\max_{\pi,s,h}\mathbb E^\pi\left[\sum_{h'=1}^H\Var_{h'}^*(s_{h'},a_{h'})\ \Bigg|\ s_h=s\right]. 
\end{align*}
\end{definition}

\begin{remark}
The maximum conditional total variance is novel in literature, as in variance-only-dependent works, $\Varmax$ is a better quantity, while in previous variance-aware gap-dependent works, researchers did not develop better approaches other than bounding total variance by $H \bbQ^*$.
By definition, $\cVarmax \le H \bbQ^*$, and our final results will scale with $\min \{ \cVarmax, H^2 \}$ after careful analysis, which can improve the dependency on $H$ by one order.
\end{remark}

\section{Main Results} \label{sec:main_results}

\subsection{Algorithm Overview: MVP} \label{sec:alg}

\begin{algorithm}
    \caption{Monotonic Value Propagation (MVP)} \label{alg:mvp}
    \begin{algorithmic}[1]
        \REQUIRE MDP $\mathcal M=(\State, \Action, H, P, R, \mu)$, learning episode number $K$, confidence parameter $\delta$, universal constants $c_1,c_2,c_3$, $\iota = \log(SAHK/\delta)$.
        \STATE Initialize: For all $(s,a,s',h)\in \State\times\Action\times\State\times[H+1]$, set $\theta_h(s,a),\kappa_h(s,a)\gets 0,n_h(s,a,s')\gets 0,n_h(s,a),Q_h(s,a),V_h(s)\gets 0$.
        \FOR{$k=1,2,\cdots,K$}
            \STATE Construct policy $\pi^k$ such that $\pi^k_h(s)=\arg\max_a Q_h(s,a)$.
            \STATE Observe trajectory $s_1^k,a_1^k,r_1^k,s_2^k,a_2^k,r_2^k,\cdots,s_h^k,a_h^k,r_h^k$.
            \FOR{$h=H, H-1, \ldots, 1$}
                \STATE $(s,a,s')\gets (s_h^k,a_h^k,s_{h+1}^k)$
                \STATE Update $n_h(s,a,s')\gets n_h(s,a,s')+1,n_h(s,a)\gets n_h(s,a)+1,\theta_h(s,a)\gets\theta_h(s,a)+r_h^k,\kappa_h(s,a)\gets\kappa_h(s,a)+(r_h^k)^2$.
                \STATE $\hat r_h(s,a)=\frac{\theta_h(s,a)}{n_h(s,a)}$
                \STATE $\hat \sigma_h(s,a)=\frac{\kappa_h(s,a)}{n_h(s,a)}$
                \STATE $\hat P_h(s,a,s')=\frac{n_h(s, a, s')}{n_h(s,a)}$
                \FOR{$(s,a)\in\State\times\Action$}
                    \STATE $b_h(s,a)\gets c_1\sqrt{\frac{\mathbb V^{s' \sim \hat P_{s,a,h}} [V_{h+1} (s')] \iota}{n_h(s,a) \lor 1}}+c_2\sqrt{\frac{(\hat\sigma_h(s,a)-(\hat r_h(s,a))^2) \iota}{n_h(s,a) \lor 1}}+c_3\frac{H \iota}{n_h(s,a) \lor 1}$
                    \STATE $Q_h(s,a)\gets\min\{\hat r_h(s,a)+\mathbb E^{s'\sim\hat P_{s,a,h}} V_{h+1}+b_h(s,a),H\}$
                    \STATE $V_h(s)\gets \max_a Q_h(s,a)$
                \ENDFOR
            \ENDFOR
        \ENDFOR
    \end{algorithmic}
\end{algorithm}

Monotonic Value Propagation (MVP, \Cref{alg:mvp}, \citet{zhang2021reinforcement,zhou2023sharp,zhang2024settling}) is a representative model-based optimistic algorithm which maintains upper bounds of $V^*$ and $Q^*$, namely $V^k$ and $Q^k$, in each episode.
The rollout policy $\pi^k$ picks the action that maximizes $Q_h^k (s, \cdot)$ at each step and updates\footnote{The original version of MVP contains a doubling mechanism to trigger updates of $V$ and $Q$ mainly to lower switching cost. Since switching cost is not central to gap-dependent analysis, we choose to update $V_h$ and $Q_h$ every episode for simplicity.} the upper bounds using Bellman equation with empirical estimates of reward and transitions:
\[
Q_h(s,a) \gets \hat r_h(s,a)+\mathbb E^{s\sim\hat P_{s,a,h}}V_{h+1}(s,a)+b_h(s,a),\quad V_h(s)\gets\max_a Q_h(s,a).
\]
Here $b_h(s,a)$ is a bonus term ensuring that $Q_h,V_h$ are upper bounds of $Q_h^*,V_h^*$ (``optimism'') with high probability.
For the proof of optimism, interested readers can refer to \citet{zhang2021reinforcement}.

\subsection{Gap-dependent Upper Bound}

Now, we present the main result of this work -- a gap-dependent regret upper bound ensured by MVP.
For the formal version, please refer to \Cref{sec:ub_proof}.

\begin{theorem}[Gap-dependent upper bound (informal)] \label{thm:upper_bound_informal}
    For any MDP instance, any episode number $K$, and $\delta>0$, MVP (\Cref{alg:mvp}) attains the following regret bound with probability at least $1-\delta$:
    \begin{align*}
        \Regret (K) \lesssim \left(\sum_{(h, s, a) \in \Zsub} \frac{H^2\log K \land \cVarmax}{\Delta_h(s,a)}+\frac{(H^2 \land \cVarmax ) |\Zopt|}{\gapmin}+SAH^4 (S \lor H) \right) \log K.
    \end{align*}
\end{theorem}

\begin{remark}
This bound contains a new notion of maximum conditional total variance (\Cref{def:unvarmax}).
Since this definition requires us to condition on any possible state, $\cVarmax$ can be as large as $\Theta(H^3)$.
However, $\cVarmax$ is bounded by $H\max_{s,a,h}\{\Var_h^*(s,a,h)\}$, so this term is still no worse than previous $O(H\mathbb Q^*)$ bounds.
Furthermore, there is a \emph{sufficient} condition to make $\cVarmax=O(H^2\log(1/\delta))$: for any policy $\pi$ and $(s,a,h)\in\State\times\Action\times[H]$, the state-action pair $(s,a)$ is not reachable at step $h$ if we sample the trajectory under $\pi$, or it is visited with probability at least $\delta$.
We can even generalize this concept to exclude the states that are difficult to reach from the definition of $\cVarmax$.
We omit this approach for simplicity.

The $H^2$ term in $H^2 \land \cVarmax$ is derived by conditioning on the event where all trajectories have bounded total variance to avoid the dependence on $\cVarmax$ when it is large. 

Our lower bound (\Cref{thm:lb}) shows that $\cVarmax$ cannot be replaced by $\Varmax$ (\Cref{def:varmax}), a quantity used by previous variance-only-dependent works.
Intuitively, $\Varmax$ can be very small as long as all states with large variance have small visiting probability, but those states can accumulate a total regret of order $\Var_h^*(s,a)/\Delta_h(s,a)$ that cannot be bounded by $\Varmax$.
The leading term matches with the lower bound modulo $\log$ factors.
Furthermore, \cite{simchowitz2019non} has shown that a $\gapmin$-dependent term is unavoidable for UCB-based algorithms.
Our coefficient of the $\gapmin$ term is also improved to worst case $O(H^2)$, better than previous worst-case factors of $H \bbQ^*$ \citep{simchowitz2019non} and $H^3 \bbQ^*$ \citep{zheng2024gap}.
\end{remark}

\subsection{Proof Sketch}

We present the high-level ideas in the proof for \Cref{thm:upper_bound_informal} here, deferring the details to \Cref{sec:ub_proof}. We assume the optimistic condition holds (see \Cref{thm:optimism}).

\paragraph{Regret decomposition.}
Following \citet{simchowitz2019non}, we define
\begin{align}
    E_h^k(s,a) := Q_h^k(s,a)-(R_h(s,a)+\mathbb E^{s'\sim P_{s,a,h}}[V_{h+1}^k(s')]) \label{eq:surplus}
\end{align}
as the surplus at $(s,a,h,k)\in\State\times\Action\times[H]\times[K]$.
Standard analysis show the regret bound
\[\Regret(K) \lesssim \mathbb E\left[\sum_{k=1}^K\sum_{h=1}^H E_h^k(s,a)\right].\]

\paragraph{Analyzing gaps and surpluses.}
Suppose that the algorithm takes action $a$ at state $s$ at episode $k$, stage $h$. By optimism, $Q_h^k(s,a)$ must be at least $V_h^*(s)=Q_h^*(s,a)+\Delta_h(s,a)$, so we have 
\begin{align*}E_h^k(s,a)+\mathbb E^{s'\sim P_{s,a,h}} [V_{h+1}^k(s')-V_{h+1}^*(s')]\geq \Delta_h(s,a).\end{align*} By recursively expanding the $V$ term, we have \begin{align}\Delta_h(s,a)\leq \mathbb E^{\pi^k}\left[\sum_{h'=h}^HE_h^k(s,a)\Bigg| (s_{h'},a_{h'})=(s,a)\right]\label{eq:lowerbound}.\end{align} That is, if the expectation of future surpluses is small, then the algorithm will avoid actions with large suboptimality gap.

The analysis of $E_h^k$ shows that (see \Cref{thm:surplusub})
\begin{align*}
E_h^k(s,a)\lesssim\sqrt{\frac{\Var_h^*(s,a)\iota}{n_h^k(s,a)}}+\underbrace{\sum_{h'\geq h}\mathbb E^{\pi^k}\left[\frac{SH\iota}{n_h^k(s_h,a_h)}\Bigg| s_h=s\right]}_{\textup{low order terms}}.
\end{align*}

We consider the restrictions of $n_h^k(s,a)$ when $\Delta_h(s,a)>0$. If the lower bound \Cref{eq:lowerbound} wrote $\Delta_h(s,a)\lesssim E_h^k(s,a)$, then $n_h^k(s,a)\lesssim\Var_h^*(s,a)\iota/\Delta_h(s,a)^2$, which would directly provide a regret bound. However, \Cref{eq:lowerbound} contains the sum all future surpluses, so we cannot directly apply this method.

We will circumvent this problem by adding \Cref{eq:lowerbound} over all $k$ and $h$. The summation of the left-hand side is $\sum_{s,a,h}\Delta_h(s,a)n_h^k(s,a)$, while the summation of the right-hand side can be shown as approximately (low-order terms discarded) $H\sum_{s,a,h}\sqrt{\Var_h^*(s,a)n_h^k(s,a)\iota}$.

This inequality has the form $\sum_{s,a,h}u_{s,a,h}n_h(s,a)\lessapprox\sum_{s,a,h}v_{s,a,h}\sqrt{n_h(s,a)}$ for some non-negative coefficients $u_{s,a,h},v_{s,a,h}$. It entails upper bounds of $n_h(s,a)$, and if we proceed with the calculations, we will recover the bound \[\Regret(K)\lesssim\sum_{s,a,h}\frac{H\Var_h^*(s,a)\iota}{\Delta_h(s,a)}+\text{(some low-order terms)}\]in \cite{simchowitz2019non} while avoiding complex calculations. In the latter steps, we will refine this method for a tighter bound.

\paragraph{Generalized weighted sum of suboptimality gaps.}
Intuitively, the previous bound is not balanced as $\Var_h^*(s,a)=\Omega(H^2)$ only happens for a small portion of $(s,a,h)$. In contrast, the summation of \Cref{eq:lowerbound} contains enough degrees of freedom for us to utilize it for a better bound. Let $w_h(s,a)$ be any set of nonnegative weights. Then the weighted sum of \Cref{eq:lowerbound} writes
\begin{align}
\sum_{s,a,h}w_h(s,a)\Delta_h(s,a)n_h^K(s,a)\lesssim\sum_{h=1}^H\sum_{k=1}^Kw_h(s_h^k,a_h^k)\sum_{h'=h}^H\mathbb E[E_{h'}^k(s_{h'}^k,a_{h'}^k)|\mathcal F_{k-1,h}],\label{eq:lb}
\end{align}
where $\mathcal F_{k-1,h}$ is the $\sigma$-field generated by first $k-1$ episodes and the first $h$ states in the $k$-th trajectory. We will choose $w_h(s,a)$ carefully to balance the contribution of each term.

\paragraph{Bounding weighted sum of surpluses.} The right-hand side of \Cref{eq:lb} needs to be manipulated carefully. Rewriting the summation order, the leading term of \Cref{eq:lb} becomes
\begin{align}
    \sum_{s',a'}\sum_{h'=1}^H\sum_{k=1}^K\sqrt{\frac{\Var_{h'}^*(s',a')\iota}{n_{h'}^k(s',a')}}\sum_{h=1}^{h'}w_h(s_h^k,a_h^k)\mathbb E[\mathbf 1\{(s_{h'}^k,a_{h'}^k)=(s',a')\}|\mathcal F_{k-1,h}].\notag
\end{align}

We will apply certain probability inequalities to approximate $\mathbb E[\mathbf 1\{(s_{h'}^k,a_{h'}^k)=(s',a')\}|\mathcal F_{k-1,h}]$ with $\mathbf 1\{(s_{h'}^k,a_{h'}^k)=(s',a')\}$ (approximation error omitted). Then, the innermost sum over $h$ contains only $w_h(s_h^k,a_h^k)$, which can be bounded by $\bar W=H^2\iota\land\cVarmax$ if we pick $w_h^k(s,a)=\Var_h^*(s,a)$. Now, the sum over $k$ is \[\sum_{n=1}^{n_{h'}^K(s',a')}\sqrt{\frac{\Var_{h'}^*(s',a')\iota}n}=O\left(\sqrt{\Var_{h'}^*(s',a')n_{h'}^K(s',a')\iota}\right),\] so by \Cref{eq:lb}, \begin{align}\sum_{s,a,h} Var_h^*(s,a)\Delta_h(s,a)n_h^K(s,a)\lesssim\bar W\underbrace{\sum_{s,a,h}\sqrt{\Var_{h}^*(s,a)n_h^K(s,a)\iota}}_{:=R}.\label{eq:final_lb}\end{align}

\paragraph{End of proof.} With similar (and simpler) arguments above, we have\[\Regret(K)\lesssim  R,\] where again, we ignore the lower-order terms.



We apply the Cauchy-Schwartz inequality to \Cref{eq:final_lb} and get \[\bar W R\cdot \left(\sum_{s,a}\sum_{h=1}^H\frac{\iota}{\Delta_h(s,a)}\right)\gtrapprox \left(\sum_{s,a,h}\sqrt{\Var_h^*(s,a)n_h^K(s,a)\iota}\right)^2=R^2,\] so we conclude that \begin{align*}&\Regret(K)\lesssim R \lessapprox \sum_{s,a,h}\frac{\bar W\iota}{\Delta_h(s,a)}.\end{align*}

\section{Gap-Dependent Lower Bound}\label{sec:lb}

In this section, we will prove the following gap-dependent regret lower bound.
It shows a separation between $\cVarmax$ and $\Varmax$, as well as the necessity of $\cVarmax$ in gap-dependent regrets.

\begin{theorem}\label{thm:lb}[Gap-dependent lower bound (informal)]
Fix $S,A,H$ and the target conditional variance $L \in [1, H^2]$.
Given a set of $S A H$ suboptimality gaps $\{\Delta_i\}$, assume that all non-zero gaps are sufficiently small.
For any algorithm, there always exists an MDP with gaps equal to $\Theta (\Delta_i)$, $\cVarmax = \Theta (L)$ but $\Varmax = O(1)$, such that
\begin{align*}
 \Regret (K) \ge \Omega\left( \sum_{i : \Delta_i > 0} \frac{L}{\Delta_i}\cdot \log K\right).
\end{align*}
\end{theorem}

\paragraph{Proof sketch.} 
We sketch the proof as follows.
For simplicity, we assume there are $4$ states $\{\mathtt{A}, \mathtt{B}, \mathtt{C},\mathtt{D} \}$ in each $h$-th layer.
 The dynamics of the four states are presented below.

\begin{itemize}
    \item $\mathtt{A}:$ There is only one action at $\mathtt{A}$, which transits the agent to $\mathtt{A}$ in the next layer with probability $1-\frac{1}{LH}$, and $\mathtt{B}$ with probability $\frac{1}{LH}$.   The reward is $0$ at $\mathtt{A}$;
    \item $\mathtt{B}:$ There are $A$ actions at $\mathtt{B}$. For each action $a$, the agent is transported to $\mathtt{C}$ with probability $\frac{1}{2}- \frac{\Delta(a_i)}{4\sqrt L}$ and $\mathtt{D}$ with probability $\frac{1}{2}+ \frac{\Delta(a_i)}{4\sqrt L}$;
    \item $\mathtt{C}:$ This state is a terminal state with reward $\sqrt L$;
    \item $\mathtt{D}:$ This state is a terminal state with reward $0$.
\end{itemize}

In this instance, the learner makes a decision only at state $\mathtt{B}$ for each layer $h$, and state $\mathtt A$ has variance $O(\frac{1}{LH}\cdot (\sqrt L)^2)=O(H^{-1})$ and state $\mathtt B$ has variance $\Theta (L)$, showing $\cVarmax = \Theta(L)$. For any strategy $\pi$, it visits $\mathtt B$ with probability $1-(1-\frac 1{LH})^H=O(L^{-1})$, So \begin{align*}
    \Varmax\leq H\cdot O(H^{-1}) + LO(L^{-1})=O(1).
\end{align*}

Clearly,  the decision problem at state $\mathtt{B}$ and layer $h$ could be viewed as a Bernoulli bandit problem.
The expected visiting count at state $\mathtt{B}$ and layer $h$ is $\Theta (K / L)$.
Let $\Regret_{h,\mathtt{B}}(K)$ be the regret by taking suboptimal actions at $\mathtt{B}$ and the $h$-th layer. Consequently, applying the classical lower bound on regret for Bernoulli bandits yields:
\begin{align*}
\lim_{K\to\infty} \frac{\Regret_{h,\mathtt{B}}(K)}{\log(K/L)}\geq \Omega\left(\sum_{a} \frac{L}{\Delta_{h}(\mathtt{B},a)}\right).
\end{align*}

Thus,
\begin{align*}
\Regret (K) \geq \sum_{h=1}^H \Regret_{h,\mathtt{B}}(K) \geq \Omega\left(\sum_{h, a} \frac{L}{\Delta_{h}(\mathtt{B},a)}\cdot \log K\right).
\end{align*}
for sufficiently large $K$.

 \paragraph{Discussion.} 
 This example shows a separation between unconditional variance $\cVarmax$ and conditional variance $\Varmax$. Even if $\Varmax=O(1)$, there can still be a regret lower bound of order $\Theta(H^2)$. In this view, our introduction to $\cVarmax$ is essential in proving gap-dependent regret bounds.
 
 We also observe that the second term $\frac{(H^2 \land \cVarmax) \abs{\Zopt}}{\gapmin}$ in our upper bound \eqref{eq:ub} is not yet matched by this lower bound.
 This could pose a significant challenge for existing optimistic algorithms, as they typically explore all potentially optimal actions, resulting in additional surplus terms.
We refer the readers to Appendix~\ref{app:lb} the full proof of Theorem~\ref{thm:lb}.

\section{Conclusion}\label{sec:conclusion}
In this paper, we study gap-dependent regret bounds for episodic MDPs and demonstrate that the Monotonic Value Propagation (MVP) algorithm \cite{zhang2024settling} achieves a tighter upper bound compared to previous works from the aspects of tighter dependence on a better variance notion, as well as reduced order of $H$.
Our analysis centers around a careful bound of the weighted sum of suboptimality gaps.
Along the way, we introduce a new notion of \emph{maximum conditional total variance} and provide a lower bound to establish its necessity as well as the tightness of the $\frac{1}{\Delta_h (s, a)}$ term.

We also acknowledge some limitations.
First, the $\frac{(H^2 \land \cVarmax) \abs{\Zopt}}{\gapmin}$ term in our upper bound does not match the lower bound of $\frac{S}{\gapmin}$ in Theorem 2.3 of \citet{simchowitz2019non}.
Improving either the upper bound or lower bound will help advancing the understanding of gap-dependent regrets.
Second, we only apply our new techniques to tabular MDPs.
For future work, we believe our analysis can be adapted to other problem settings (e.g., linear MDPs \citep{wagenmaker2022instance} and MDPs with general function approximation) to derive tighter gap-dependent regret bounds.

\newpage

\section*{Acknowledgement}
SSD acknowledges the support of NSF DMS 2134106, NSF CCF 2212261, NSF IIS 2143493, NSF IIS 2229881, Alfred P. Sloan Research Fellowship, and Schmidt Sciences AI 2050 Fellowship.
RZ and MF acknowledge the support of NSF TRIPODS II DMS-2023166. The work of MF was supported in part by awards NSF CCF 2212261 and NSF CCF 2312775.

\bibliographystyle{plainnat}
\bibliography{ref.bib}

\newpage
\appendix

\section{Notations and Technical Lemmas}

\subsection{Notations}

We list notations in \Cref{tab:params_mdp,tab:values_algo,tab:other_notations}.

\begin{table}[ht]
    \centering
    \begin{tabular}{|c|c|}
        \hline
         $\State,S=|\State|$ & State space and its size\\
         $\Action, A=|\Action|$ & Action space and its size\\
         $H$ & Horizon \\
         $K$ & Learning episodes \\
         $s, s'$ & States in $\State$ \\
         $a, a'$ & Actions in $\Action$ \\
         $h, h', h^*$ & Horizon numbers \\
         $k, k'$ & Indices of learning episode \\
         $P_{s,a,h}$ & Transition probability \\
         $R_{s,a,h}$ & Distribution of rewards \\
         $\mu$ & Distribution of beginning state \\
         $r_h(s,a)$ & Expected reward \\
         $\pi$ & Policy \\
         $\pi_h(s)$ & Action that policy $\pi$ takes at state $s$, step $h$ \\
         $V_h^\pi(s),V_h^*(s)$ & $V$-function of policy $\pi$ and of optimal policy, respectively \\
         $Q_h^\pi(s,a),Q_h^*(s,a)$ & $Q$-function of policy $\pi$ and of optimal policy, respectively \\
         $\Var_h^*(s,a)$ & Variance at state $s$, action $a$, and step $h$ \\
         $\Varmax$ & Maximum unconditional variance \\
         $\cVarmax$ & Maximum conditional variance \\
         $\Delta_h(s,a)$ & Suboptimality gap \\
         $\gapmin$ & Minimal nonzero suboptimality gap \\
         $\Zsub$ & Set of suboptimal actions \\
         $\Zopt$ & Set of optimal actions \\
         \hline
    \end{tabular}
    \caption{Parameters of MDP}
    \label{tab:params_mdp}
\end{table}

\begin{table}[ht]
    \centering
    \begin{tabular}{|c|c|}
        \hline
        $s_h^k,a_h^k,r_h^k$ & States, actions, and rewards observed in the $k$-th episode \\
        $V_h^k(s)$ & $V_h$ of the algorithm before the $k$-th episode \\
        $Q_h^k(s,a)$ & $Q_h$ of the algorithm before the $k$-th episode \\
        $\hat r_h^k(s,a)$ & Estimation of $r_h(s,a)$ before the $k$-th episode \\
        $\hat\sigma_h^k(s,a)$ & Estimation of $\sigma_h(s,a)$ before the $k$-th episode \\
        $\hat P_{s,a,h}^k$ & Estimation of $P_{s,a,h}$ before the $k$-th episode \\ 
        $\hat n_h^k(s,a)$ & Visitation count at $(s,a,h)$ before the $k$-th episode \\
        $b_h^k(s,a)$ & Bonus term in the $k$-th episode \\
        $\pi^k$ & The policy at the $k$-th episode \\
        \hline
    \end{tabular}
    \caption{Values used in the algorithm}
    \label{tab:values_algo}
\end{table}

\begin{table}[ht]
    \centering
    \begin{tabular}{|c|c|}
        \hline
        $[n]$ & Set $\{1,2,\cdots,n\}$ \\
        $\Delta^B$ & Set of distribution functions over set $B$ \\
        $x\land y$ & $\min\{x,y\}$ \\
        $x\lor y$ & $\max\{x,y\}$ \\
        $\mathbf 1\{\varphi\}$ & Indicator function of $\varphi$, i.e. $1$ if $\varphi$ is true and $0$ otherwise \\
        $\clip{a}{\varepsilon}$ & $a\mathbf 1\{a\geq\varepsilon\}$ \\
        $\delta$ & Acceptable error probability \\
        $E_h^k(s,a)$ & Surplus; $Q_h^k(s,a)-r_h(s,a)-\mathbb E^{s'\sim P_{s,a,h}}[V_{h+1}^k(s')]$ \\
        $\bar E_h^k(s,a)$ & Clipped surplus \\
        $\iota$ & $\log(SAHK/\delta)$ \\
        $w_h(s,a)$ & Weights used in analysis \\
        $\bar W$ & $160H^2\log(4K(H+1)/\delta)\land \cVarmax$ \\
        $\Regret$ & Total regret \\
        $\mathcal F_k$ & $\sigma$-field generated by the first $k-1$ episodes of the algorithm \\
        $\mathcal F_{k,h}$ & $\sigma$-field generated by the first $k-1$ episodes and the first $h$ steps in the $k$-th episode \\
        $\sum_s,\sum_a,\sum_{s,a}$ & $\sum_{s\in\State},\sum_{a\in\Action},\sum_{s\in\State,a\in\Action}$, respectively \\
        $\mathbb E^{x\sim X},\mathbb V^{x\sim X}$ & Expectation when $x$ is sampled from distribution $X$ \\
        $\mathbb P^\pi,\mathbb E^{\pi}$ & Probability and expectation over a trajectory when following policy $\pi$ \\
        \hline
    \end{tabular}
    \caption{Other notations}
    \label{tab:other_notations}
\end{table}

\subsection{Technical Lemmas}
\begin{lemma}[Bennett's inequality, Theorem 3 in \citet{maurer2009empirical}]
    \label{thm:bennett}
    Let $X_1,X_2,\cdots,X_n$ be i.i.d. random variables with values $[0,a](a>0)$ and let $\delta>0$. Then, 
    \[\mathbb P\left[\left| \mathbb{E}[X_1] - \frac1n\sum_{i=1}^nX_i \right| > \sqrt{\frac{2\mathbb V[X_1]\log(2/\delta)}{n}} + \frac{a\log(2/\delta)}{n} \right] < \delta.\]
\end{lemma}

\begin{lemma}[Freedman's inequality, Lemma 10 in \cite{zhang2020almost}]
    \label{thm:freedman}
    Let $(X_n)_{n\geq 1}$ be a martingale difference sequence (i.e., $\mathbb E[X_n|\mathcal{F}_{n-1}]=0$ for all $n\geq 1$, where $\mathcal{F}_k=\sigma(X_1,X_2,\cdots,X_k)$) such that $|X_n|\leq a$ for some $a>0$ and for all $n\geq 1$. Let $V_n=\sum_{k=1}^n\mathbb E[X_k^2|\mathcal F_{k-1}]$ for $n\geq 0$. Then, for any positive integer $n$, and any $\varepsilon,\delta>0$, we have \[ \mathbb{P}\left[ \left|\sum_{i=1}^nX_i\right| \geq 2\sqrt{V_n\log(1/\delta)} + 2\sqrt{\varepsilon\log(1/\delta)} + 2a\log(1/\delta)\right] \leq 2(na^2\varepsilon^{-1}+1)\delta.\]
\end{lemma}

\begin{lemma}[Lemma 10 in \cite{zhang2022horizon}]
    \label{lem:martingale}
Let $X_1, X_2, \ldots$ be a sequence of random variables taking values in $[0, l]$.
Define $\cF_k = \sigma (X_1, X_2, \ldots, X_{k - 1})$ and $Y_k = \bbE [X_k\ |\ \cF_k]$ for $k \ge 1$.
For any $\delta > 0$, we have that
\begin{align*}
    &\bbP \left[ \exists n, \sum_{k = 1}^n X_k \ge 3 \sum_{k = 1}^n Y_k + l \ln (1 / \delta) \right] \le \delta, \\
    &\bbP \left[ \exists n, \sum_{k = 1}^n Y_k \ge 3 \sum_{k = 1}^n X_k + l \ln (1 / \delta) \right] \le \delta.
\end{align*}
\end{lemma}

\begin{lemma}[Lemma F.5 in \citet{simchowitz2019non}]
    \label{thm:vardiff}
    Let $X,Y$ be two random variables defined on the same probability space. Then \[|\sqrt{\mathbb V[X]}-\sqrt{\mathbb V[Y]}|\leq\sqrt{\mathbb E[(X-Y)^2]}.\]
\end{lemma}

\begin{lemma}[Lemma B.5 in \citet{simchowitz2019non}]
    \label{thm:distributeclipping}
    Let $a_1,a_2,\cdots,a_m$ be a sequence of nonnegative reals and $\varepsilon>0$. Then, \[\clip{\sum_{i=1}^ma_i}{\varepsilon}\leq 2\sum_{i=1}^m\clip{a_i}{\frac{\varepsilon}{2m}}.\]
\end{lemma}

\subsection{Model errors}
Our analysis will mostly be based on the success of following inequalities. 

\begin{lemma}[Good events]
    \label{thm:concentration}
    Let $\iota=\log(SAHK/\delta)$. With probability at least $1-10\delta$, the following inequalities hold for all $s,a,s',h,k$:

    \[|\hat r_h^k(s,a)-r_h(s,a)|\leq\sqrt{\frac{2\mathbb V^{r'\sim R_{s,a,h}}[r']\iota}{n_h^k(s,a)}}+\frac{H\iota}{n_h^k(s,a)},\]

    \[|\hat P_{s,a,h}^k(s')-P_{s,a,h}(s')|\leq\sqrt{\frac{2P_{s,a,h}(s')\iota}{n_h^k(s,a)}}+\frac{\iota}{n_h^k(s,a)},\]
    
    \[|\mathbb E^{s'\sim \hat P_{s,a,h}^k}[V_{h+1}^*(s')]-\mathbb E^{s'\sim P_{s,a,h}^k}[V_{h+1}^*(s')]|\leq\sqrt{\frac{2\mathbb V^{s'\sim P_{s,a,h}}V_{h+1}^*(s')\iota}{n_h^k(s,a)}}+\frac{H\iota}{n_h^k(s,a)},\]

    \[\sqrt{\mathbb V^{s'\sim \hat P_{s,a,h}^k}[V_{h+1}^*(s')]}-\sqrt{\mathbb V^{s'\sim P_{s,a,h}}[V_{h+1}^*(s')]}\leq H\sqrt{\frac{2\iota}{n_h^k(s,a)-1}}.\]

    \[\sqrt{\hat\sigma_h^k(s,a)-(\hat r_h^k(s,a))^2}-\sqrt{\mathbb V^{r'\sim R_{s,a,h}}[r']}\leq H\sqrt{\frac{2\iota}{n_h^k(s,a)-1}}.\]
\end{lemma}

\begin{proof}
    The first three inequalities can be derived from Theorem~\ref{thm:bennett}, Theorem~\ref{thm:freedman}. The last two inequalities are adapted from Theorem 10 in \cite{maurer2009empirical}.
\end{proof}

\begin{lemma}
    \label{thm:p_concentrate}
    Let $V$ be a function defined on $\State$. Conditioned on the success of \Cref{thm:concentration}, 
    \begin{align*}|\mathbb E^{s'\sim\hat P_{s,a,h}}[V(s')]-\mathbb E^{s'\sim P_{s,a,h}}[V(s')]| \leq \sqrt{\frac{2S\mathbb E^{s'\sim P_{s,a,h}}[V(s')^2]\iota}{n_h^k(s,a)}}+\frac{\max_{s\in\State}|V(s)|S\iota}{n_h^k(s,a)}.
    \end{align*}
\end{lemma}

\begin{proof}
    Let $M=\max_{s\in\State}|V(s)|$. Then, by \Cref{thm:concentration},
    \begin{align*}
        &|\mathbb E^{s'\sim\hat P_{s,a,h}}[V(s')]-\mathbb E^{s'\sim P_{s,a,h}}[V(s')]| \\
        = & \left|\sum_{s'\in \State}(\hat P_{s,a,h}^k(s')-P_{s,a,h}(s'))V(s')\right|\\
        \leq & \sum_{s'\in\State}|V(s')|\left(\sqrt{\frac{2P_{s,a,h}(s')\iota}{n_h^k(s,a)}}+\frac{\iota}{n_h^k(s,a)}\right)\\
        \leq &\sqrt{\left(\sum_{s'\in\State}P_{s,a,h}(s')V(s')^2\right)\left(\sum_{s'\in\State}\frac{2\iota}{n_h^k(s,a)}\right)}+\frac{MS\iota}{n_h^k(s,a)}\\
        = & \sqrt{\mathbb E^{s'\sim P_{s,a,h}}[V(s')^2]\cdot\frac{2S\iota}{n_h^k(s,a)}}+\frac{MS\iota}{n_h^k(s,a)}.
    \end{align*}
\end{proof}

\subsection{Variance bounds}

\begin{corollary}
    \label{thm:Vexpect}
    Let $\pi$ be any fixed policy. For any $h\in H$ and $s\in \State$, we have \[\mathbb E^\pi\left[\sum_{h'=h}^H\Var_h^*(s_{h'},a_{h'})\Bigg|s_h=s\right]\leq 160H^2(\log(4(H+1))+1).\]
\end{corollary}

\begin{proof}
    Recall that $\mathbb E^{s'\sim P_{s,a,h}}[V_{h+1}^*(s')]=Q_h^*(s,a)-r_h(x,a)\leq V_h^*(s)-r_h(s,a)$, so
    \begin{align}
        & \mathbb E^\pi\left[\sum_{h'=h}^H\Var_h^*(s_{h'},a_{h'})\Bigg|s_h=s\right]  \notag \\
        = & \mathbb E^\pi\left[\sum_{h'=h}^H \mathbb V^{r'\sim R_{s_{h'},a_{h'},h'}}[r'] + \sum_{h=h'}^{H}\mathbb{V}^{s'\sim P_{s_{h'},a_{h'},h'}}[V_{h'+1}^*(s')] \Bigg| s_h=s\right] \notag \\
        \leq & \mathbb E^\pi\left[\sum_{h'=h}^H\mathbb (r_{h'}'-r_{h'}(s_{h'},a_{h'}))^2+\sum_{h'=h}^H(V_{h'+1}^*(s_{h'+1})-V_{h'}^*(s_{h'},a_{h'})+r_{h'}(s_{h'}))^2\Bigg| s_h=s\right] \notag \\
        \leq & \mathbb E^\pi \left[\left(\sum_{h'=h}^H(r_{h'}'+V_{h'+1}^*(s_{h'+1})-V_{h'}^*(s_{h'}))\right)^2 \Bigg| s_h=s \right] \label{eq:independence} \\
        \leq & \mathbb E^\pi \left[\left(\sum_{h=h'}^{H} r_{h'}'-V_h^*(s_h)\right)^2 \Bigg| s_h=s\right] \leq H^2\notag,
    \end{align} where \Cref{eq:independence} is because of independence and that \[\mathbb E^{s'\sim P_{s_{h'},a_{h'},h}}[V_{h'+1}^*(s')-V_{h'}^*(s_{h'},a_{h'})+r_{h'}(s_{h'})]\leq 0=\mathbb E^{r'\sim R_{s_{h'},a_{h'},h}}[r'-r_{h'}(s_{h'},a_{h'})].\]
\end{proof}

\begin{lemma}[Lemma 42, \cite{zhou2023sharp}]
    \label{thm:variancebound}
    Let $\pi$ be any fixed policy. For any $\delta>0$, \[\mathbb P^\pi\left[\sum_{h'=h}^H\Var_h^*(s_{h'},a_{h'})\geq 160H^2\log(4(H+1)/\delta)\Bigg| s_h=s\right]\leq \delta.\]
\end{lemma}

\begin{proof}
    We have 
    \begin{align}
        & \sum_{h'=h}^H\Var_h^*(s_{h'},a_{h'})=\sum_{h'=h}^H\mathbb V^{r'\sim R_{s_{h'},a_{h'},h'}}[r'] + \sum_{h'=h}^H\mathbb V^{s'\sim P_{s_{h'},a_{h'},h'}}[V_{h'+1}^*(s')] \notag \\
        = & \sum_{h'=h}^H\mathbb E^{s'\sim P_{s_{h'},a_{h'},h'}}[V_{h'+1}^*(s')^2]-\sum_{h'=h}^H(Q_{h'}^*(s_{h'},a_{h'})-r_{h'}(s_{h'},a_{h'}))^2+\sum_{h'=h}^HHr_{h'}(s_{h'},a_{h'}) \notag \\
        \leq & \sum_{h'=h}^H(\mathbb E^{s'\sim P_{s_{h'},a_{h'},h'}}[V_{h'+1}^*(s')^2]-V_{h'+1}^*(s_{h'+1})^2) \notag \\
        & +\sum_{h'=h}^H(V_{h'}^*(s_{h'})^2-(Q_{h'}^*(s_{h'},a_{h'})-r_{h'}(s_{h'},a_{h'}))^2)+H^2 \notag \\
        \leq & 2\sqrt{2\sum_{h'=h}^H\mathbb V^{s'\sim P_{s_{h'},a_{h'},h'}}[V_{h'+1}^*(s_{h'+1})^2]\log(1/\delta)}+2\sqrt{H^4\log(1/\delta)}+2H^2\log(1/\delta)\label{eq:freedman1} \\
        & +2H\sum_{h'=h}^H(V_{h'}^*(s_{h'})-Q_{h'}^*(s_{h'},a_{h'})+r_{h'}(s_{h'},a_{h'}))+H^2 \notag \\
        \leq & 4H\sqrt{2\sum_{h'=h}^H\mathbb V^{s'\sim P_{s_{h'},a_{h'},h'}}[V_{h'+1}^*(s')^2])\log(1/\delta)}+5H^2\log(1/\delta)+2H\cdot V_{h}^*(s_h) \notag \\
        &\label{eq:freedman2} + 4H\sqrt{2\sum_{h'=h}^H\mathbb V^{s'\sim P_{s_{h'},a_{h'},h'}}[V_{h'+1}^*(s')^2])\log(1/\delta)}+4H\sqrt{H^2\log(1/\delta)}+4H^2\log(1/\delta) \\
        \leq & 8H\sqrt{2\sum_{h'=h}^H\Var_{h'}^*(s_{h'},a_{h'})} + 15H^2\log(1/\delta), \notag
    \end{align}
    where \Cref{eq:freedman1,eq:freedman2} holds with probability $1-2(H+1)\delta$ each by \Cref{thm:freedman}. Thus, by solving the quadratic equation, \[\mathbb P^\pi\left[\sum_{h'=h}^H\Var_h^*(s_{h'},a_{h'})\geq 160H^2\log(1/\delta)\Bigg| s_h=s\right]\leq 1-4(H+1)\delta.\]
    The proof is finished with rescaling $\delta$.
\end{proof}


\section{Proof of main theorem} \label{sec:ub_proof}

We begin by choosing the universal constants in the algorithm as $c_1=c_2=2,c_3=10$. 

\subsection{Clipping surpluses}
Existing analysis of MDP already shows that $Q_h^k$ and $V_h^k$ are upper bounds of $Q_h^*$ and $V_h^*$ with high probability as expected:
\begin{lemma}
    \label{thm:optimism}
    With probability at least $1-4\delta$, for all $s,a,h,k\in\State \times \Action\times [H]\times[K]$, 
    \[Q_h^k(s,a)\geq Q_h^*(s,a),V_h^k(s)\geq V_h^*(s).\]
\end{lemma}

\begin{proof}
    The proof is almost the same as Lemma 8 in \cite{zhang2024settling} with necessary modifications for our constant choices. Since $c_3=10\geq 4=c_1^2$, the monotonic function can be constructed as \[f_{P,n}(v):=\mathbb E^{s\sim P}[v(s)]+\max\left\{2\sqrt{\frac{\mathbb V^{s\sim P}[v(s)]\iota}{n}},\frac{4H\iota}{n}\right\}.\]
\end{proof}

We define clipped surpluses as \begin{align}
        \bar E_h^k(s_h,a_h)=\clip{E_h^k(s_h,a_h)}{c_4\gapmin\max\left\{\frac{\Var_h^*(s,a)}{\min\{H^2,\cVarmax\}}+\frac{1}{H}\right\}}. \label{eq:def_clippedsurplus}
    \end{align} Also, we recursively define \[\bar Q_{H+1}^k(s,a)=\bar V_{H+1}^k(s)=0,\] \[\bar Q_h^k(s,a)=r_h^k(s,a)+\bar E_h^k(s,a)+\mathbb E^{s'\sim P_{s,a,h}}[\bar V_{h+1}^k(s)],\bar V_h^k(s)=\max_a\bar Q_h^k(s,a)\] for $h=H,H-1,\cdots,1$, and $\bar Q_0^k=\bar V_0^k=\mathbb E^{s'\sim\mu}[\bar V_1^k(s')]$.

\begin{lemma}
    \label{thm:smallclip}
    \[\bar{{V}}_h^k(s)\geq{V}_h^{k}(s)-\frac{\gapmin}{3}.\]
\end{lemma}

\begin{proof}
    We have $\bar{ E}_h^k(s,a)\geq  E_h^k(s,a)-\frac{\gapmin\Var_h^*(s,a)}{6(H^2\land\cVarmax)}-\frac{\gapmin}{6H}$ for any pair of $s,a,h$. Thus, 
    \[\begin{aligned}
        &\bar { V}_h^k(s)- V_h^{\pi^k}(s)\\
        =&\mathbb E^{\pi^k}\left[\sum_{h'= h}^H\bar{ E}_h^k(s_h,a_h)\Bigg|s_h=s\right]\\
        \geq&\mathbb E^{\pi^k}\left[\sum_{h'= h}^H{ E}_h^k(s_h,a_h)\Bigg| s_h=s\right]-\frac{\gapmin}{6(H^2\land\cVarmax)}\mathbb E^{\pi^k}\left[\sum_{h'=h}^H\Var_h^*(s_h,a_h)\Bigg| s_h=s\right]-\sum_{h'=h}^H\frac{\gapmin}{6H}\\
        \geq&  V_h^k(s)- V_h^{\pi^k}(s)-\frac{\gapmin}{3},
    \end{aligned}\]

    where the last line is due to Theorem~\ref{thm:Vexpect} and definition of $\cVarmax$.
\end{proof}

This lemma links the half-clipped values $\bar{ V}_h^k$ with the optimal values $ V_h^*$.
\begin{lemma}
    \label{thm:gaplb}
    Conditioned on success of \Cref{thm:optimism}, for any state $s\in \State$ and $h\in[H]$,
    \[{V}_h^*(s) - {V}_h^{\pi^k}(s) \leq \frac32(\bar{V}_h(s)-{V}_h^{\pi^k}(s)).\]
\end{lemma}

\begin{proof}
    The first step in the proof is to recursively expand both sides at all states where an optimal action is taken. Specifically, we let \[\mathcal E_{h^*}=\{\pi_{h'}^k(s_{h'})=a_{h'},h'=h,h+1,\cdots,h^*\},\] and $\mathcal E_{h^*}-\mathcal E_{h^*+1}$ as the set of the trajectories in $\mathcal E_{h^*}$ but not in $\mathcal E_{h^*+1}$ (that is, those trajectories where the first suboptimal action after the $h$-step is at the $(h^*+1)$-th step). Since trajectories are sampled with policy $\pi^k$, $\mathcal E_h$ is the set of all trajectories with $s_h=s$.
    
    We hope to claim that
    \begin{align}
        V_h^*(s)-V_h^{\pi^k}(s)& =\sum_{h'=h+1}^{h^*}\mathbb E^{\pi^k}[\mathbf 1\{\mathcal E_{h'-1}-\mathcal E_{h'}\}(V_{h'}^*(s_{h'})-V_{h'}^{\pi^k}(s_{h'}))|s_h=s]\label{eq:Vstar} \\
        & +\mathbb E^{\pi^k}[\mathbf 1\{\mathcal E_{h^*}\}(V_{h^*+1}^*(s_{h^*+1})-V_{h^*+1}^{\pi^k}(s_{h^*+1})|s_h=s]
    \end{align} and \begin{align}
        \bar V_h^k(s)-V_h^{\pi^k}(s)& \geq\sum_{h'=h+1}^{h^*}\mathbb E^{\pi^k}[\mathbf 1\{\mathcal E_{h'-1}-\mathcal E_{h'}\}(\bar V_{h'}^k(s_{h'})-V_{h'}^{\pi^k}(s_{h'}))|s_h=s]\label{eq:Vbar} \\
        & +\mathbb E^{\pi^k}[\mathbf 1\{\mathcal E_{h^*}\}(\bar V_{h^*+1}^k(s_{h^*+1})-V_{h^*+1}^{\pi^k}(s_{h^*+1}))|s_h=s].
    \end{align}
    These claims are proved by induction on $h^*$ and expanding the last term on event $\mathcal E_{h^*+1}$. For \Cref{eq:Vstar}, we have \begin{align*}
        &V_{h^*}^*(s_{h^*})-V_{h^*}^{\pi^k}(s_{h^*}) \\
        =&Q_{h^*}^*(s_{h^*},\pi_{h^*}^k(s_{h^*}))-Q_{h^*}^{\pi^k}(s_{h^*},\pi_{h^*}^k(s_{h^*})) \\
        =&\mathbb E^{s'\sim P_{s,\pi_{h^*},k}}[V_{h^*+1}^*(s')-V_{h^*+1}^*(s')]
    \end{align*} when the trajectory is in $\mathcal E_{h^*+1}$, and for \Cref{eq:Vbar}, we have \begin{align*}
        &\bar V_{h^*}^k(s)-V_{h^*}^{\pi^k}(s)=\bar Q_{h^*}^k(s,\pi_{h^*}^k(s))-Q_{h^*}^{\pi^k}(s,\pi_{h^*}^k(s)) \\
        =&\bar E_{h^*}^k(s,a)+\mathbb E^{s'\sim P_{s,\pi_{h^*}^k(s),k}}[\bar V_{h^*+1}^k(s')-\bar V_{h^*+1}^{\pi^k}(s')] \\
        \geq & \mathbb E^{s'\sim P_{s,\pi_{h^*}^k(s),k}}[\bar V_{h^*+1}^k(s')-\bar V_{h^*+1}^{\pi^k}(s')].
    \end{align*}

    We will use \Cref{eq:Vstar} and \Cref{eq:Vbar} when $h^*=H$. In this case, the last lines are both zero, so it suffices to show that \[\frac 32(\bar V_{h'}^k(s_{h'})-V_{h'}^{\pi^k}(s_{h'}))\geq V_{h'}^*(s_{h'})-V_{h'}^{\pi^k}(s_{h'})\] on $\mathcal E_{h'-1}-\mathcal E_{h'}$. In fact, since the trajectory is sampled from $\pi^k$, and since $a_{h'}$ is suboptimal, we have that \begin{align*}
        \bar V_{h'}^k(s_{h'})\geq V_h^k(s_{h'})-\frac{\gapmin}{3}=Q_{h'}^k(s_{h'},a_{h'})-\frac{\gapmin}{3} 
        \geq V_{h'}^*(s_{h'})-\frac{\Delta_{h'}(s_{h'},a_{h'})}{3}\geq\frac23V_{h'}^*(s_{h'}),
    \end{align*} where the last inequality is because $\Delta_{h'}(s_{h'},a_{h'})=V_{h'}^{\pi^k}(s_{h'})- Q_{h'}^*(s_{h'},a_{h'})\leq V_{h'}^{\pi^k}(s_{h'})$.\
\end{proof}

\begin{lemma}
    \label{thm:gaplbsurplus}
    Conditioned on success of \Cref{thm:optimism}, if $a=\pi_h^k(s)$, then \[\Delta_h(s,a)\leq \frac32\sum_{h'=h}^H\mathbb E^{\pi^k}[\bar E_{h'}^k(s_{h'},a_{h'})|(s_h,a_h)].\]
\end{lemma}

\begin{lemma}
    \label{thm:regretub}
    Conditioned on success of \Cref{thm:optimism}, \[V_0^*-V_0^{\pi^k}\leq \frac32\sum_{h=1}^H\mathbb E^{\pi^k}[\bar E_{h}^k(s_{h},a_{h})].\]
\end{lemma}

\begin{proof} 
    We prove \Cref{thm:gaplbsurplus,thm:regretub} together. By recursively applying \[\bar V_h^k(s)-V_h^{\pi^k}(s)=\bar{{E}}_h^k(s,\pi_h^k(s))+\mathbb{E}^{s'\sim P_{s,\pi_h^k(s),h}}[\bar{ V}_{h+1}^k(s)- V_{h+1}^{\pi^k}(s)],\] we have \[\bar V_h^k(s)-V_h^{\pi^k}(s)=\sum_{h'=h}^H\mathbb E^{\pi^k}\left[\bar E_{h'}^k(s_{h'},a_{h'})|s_h=s\right].\] Then we use \Cref{thm:gaplb} and \[\Delta_h(s,a)=V_h^*(s)-Q_h^*(s,a)\leq V_h^*(s)-V_h^{\pi^k}(s),\quad  V_0^*-V_0^{\pi^k}=\mathbb E^{\pi^k}[V_1^*(s_1)-V_1^{\pi^k}(s_1)],\] for \Cref{thm:gaplbsurplus,thm:regretub}, respectively.
\end{proof}

\subsection{Estimating Surpluses}\label{sec:surplus}

\begin{lemma}
    \label{thm:bonus}
    Conditioned on success of \Cref{thm:concentration}, \[b_h^k(s,a)\leq \frac{2}{H}\mathbb E^{s'\sim \hat P_{s,a,h}^k}[(V_{h+1}^k(s')-V_{h+1}^*(s'))^2]+2\sqrt{\frac{2\Var_h^*(s,a)\iota}{n_h^k(s,a)}}+\frac{20H\iota}{n_h^k(s,a)}.\]
\end{lemma}

\begin{proof}
    Recall that our choice of bonus in the algorithm is \[b_h^k(s,a)= 2\sqrt{\frac{\mathbb V^{s' \sim \hat P_{s,a,h}^k} [V_{h+1} (s')] \iota}{n_h^k(s,a) }}+2\sqrt{\frac{(\hat\sigma_h^k(s,a)-(\hat r_h^k(s,a))^2) \iota}{n_h^k(s,a)}}+\frac{10H \iota}{n_h^k(s,a)}.\] Since the last term is at least $H$ if $n_h^k(s,a)=1$, it suffices to consider $n_h^k(s,a)\geq 2$. The first term can be bounded using \begin{align*}
         &\sqrt{\mathbb V^{s' \sim \hat P_{s,a,h}^k} [V_{h+1}^k(s')]}= \left(\sqrt{\mathbb V^{s' \sim \hat P_{s,a,h}^k} [V_{h+1}^k(s')]}-\sqrt{\mathbb V^{s' \sim \hat P_{s,a,h}^k} [V_{h+1}^*(s')]}\right) \\ 
         & +\left(\sqrt{\mathbb V^{s' \sim \hat P_{s,a,h}^k} [V_{h+1}^*(s')]}-\sqrt{\mathbb V^{s' \sim P_{s,a,h}^k} [V_{h+1}^*(s')]}\right)+\sqrt{\mathbb V^{s' \sim P_{s,a,h}^k} [V_{h+1}^*(s')]} \\
        \leq & \sqrt{\mathbb E^{s'\sim\hat P_{s,a,h}^k}[(V_{h+1}^k(s')-V_{h+1}^*(s'))^2]}+H\sqrt{\frac{2\iota}{n_h^k(s,a)-1}}+\sqrt{\mathbb V^{s'\sim P_{s,a,h}^k}[V_{h+1}^*(s')]}
    \end{align*} by \Cref{thm:vardiff,thm:concentration}, so \begin{align*}
        & \sqrt{\frac{\mathbb V^{s' \sim \hat P_{s,a,h}^k} [V_{h+1} (s')] \iota}{n_h^k(s,a) }} \\
        \leq &\sqrt{\mathbb E^{s'\sim\hat P_{s,a,h}^k}[(V_{h+1}^k(s')-V_{h+1}^*(s'))^2]\cdot\frac{\iota}{n_h^k(s,a)}}+\frac{2H\iota}{n_h^k(s,a)}+\sqrt{\frac{\mathbb V^{s'\sim P_{s,a,h}^k}[V_{h+1}^*(s')]\iota}{n_h^k(s,a)}} \\
        \leq &\frac{1}{H}\mathbb E^{s'\sim \hat P_{s,a,h}^k}[(V_{h+1}^k(s')-V_{h+1}^*(s'))^2]+\sqrt{\frac{\mathbb V^{s'\sim P_{s,a,h}^k}[V_{h+1}^*(s')]\iota}{n_h^k(s,a)}}+\frac{3H\iota}{n_h^k(s,a)}.
    \end{align*}

    The second term of $b_h^k(s,a)$ can easily be bounded by \Cref{thm:concentration} as \[\sqrt{\frac{(\hat\sigma_h^k(s,a)-(\hat r_h^k(s,a))^2) \iota}{n_h^k(s,a)}}\leq\sqrt{\frac{\mathbb V^{r'\sim R_{s,a,h}}[r']\iota}{n_h^k(s,a)}}+\frac{2H\iota}{n_h^k(s,a)}.\] Thus \[b_h^k(s,a)\leq \frac{2}{H}\mathbb E^{s'\sim \hat P_{s,a,h}^k}[(V_{h+1}^k(s')-V_{h+1}^*(s'))^2]+2\sqrt{\frac{2\Var_h^*(s,a)\iota}{n_h^k(s,a)}}+\frac{20H\iota}{n_h^k(s,a)}.\]
\end{proof}

\begin{lemma}
    \label{thm:vtail}
    Conditioned on success of \Cref{thm:concentration}, 
    \begin{align*}
        V_h^k(s)-V_h^*(s)\leq \mathbb E^{\pi^k}\left[\sum_{h'=h}^HH\land 22H\sqrt{\frac{S\iota}{n_{h'}^k(s_{h'},a_{h'})}}\Bigg | s_h=s\right]
    \end{align*}
\end{lemma}

\begin{proof}
    We begin by decomposing $V_h^k(s')-V_h^*(s')$ as follows: \begin{align*}
        & V_h^k(s)-V_h^*(s)\leq V_h^k(s)-Q_h^*(s,\pi_h^k(s)) \\
        = & \hat r_h^k(s,a)+b_h^k(s,a)+\mathbb E^{s'\sim\hat P_{s,a,h}}V_{h+1}^k(s')-r_h(s,a)-\mathbb E^{s'\sim P_{s,a,h}}V_{h+1}^*(s') \\
        = &(\hat r_h^k(s,a)-r_h(s,a))+(\mathbb E^{s'\sim\hat P_{s,a,h}}(V_{h+1}^k(s')-V_{h+1}^*(s'))-\mathbb E^{s'\sim P_{s,a,h}}(V_{h+1}^k(s')-V_{h+1}^*(s'))) \\
        & + (\mathbb E^{s'\sim \hat P_{s,a,h}}V_{h+1}^*(s')-\mathbb E^{s'\sim P_{s,a,h}}V_{h+1}^*(s')) + \mathbb E^{s'\sim P_{s,a,h}}(V_{h+1}^k(s')-V_{h+1}^*(s'))+b_h^k(s,a).
    \end{align*}

    By \Cref{thm:concentration,thm:p_concentrate} (with $V=V_{h+1}^k-V_{h+1}^*$) and definition of $b_h^k(s,a)$, we conclude \begin{align*}
    & V_h^k(s)-V_h^*(s)\leq \mathbb E^{s'\sim P_{s,a,h}}[V_{h+1}^k(s')-V_{h+1}^*(s')]+\left(\sqrt{\frac{2\mathbb V^{r'\sim R_{s,a,h}}[r']\iota}{n_h^k(s,a)}}+\frac{H\iota}{n_h^k(s,a)}\right) \\
    & +\left(\sqrt{\frac{2S\mathbb E^{s'\sim P_{s,a,h}}[(V_{h+1}^k(s')-V_{h+1}^*(s'))^2]\iota}{n_h^k(s,a)}}+\frac{SH\iota}{n_h^k(s,a)}\right) \\
    & +\left(\sqrt{\frac{2\mathbb V^{s'\sim P_{s,a,h}}[V_{h+1}^*(s')^2]\iota}{n_h^k(s,a)}}+\frac{H\iota}{n_h^k(s,a)}\right) \\
    & +\left(2\sqrt{\frac{\mathbb V^{s' \sim \hat P_{s,a,h}^k} [V_{h+1} (s')] \iota}{n_h^k(s,a) }}+2\sqrt{\frac{(\hat\sigma_h^k(s,a)-(\hat r_h^k(s,a))^2) \iota}{n_h^k(s,a)}}+\frac{10H \iota}{n_h^k(s,a)}\right) \\
    \leq & \mathbb E^{s'\sim P_{s,a,h}}[V_{h+1}^k(s')-V_{h+1}^*(s')]+(3\sqrt 2+4)H\sqrt{\frac{S\iota}{n_h^k(s,a)}}+\frac{13SH\iota}{n_h^k(s,a)}.
    \end{align*}

    If $S\iota\leq {n_h^k(s,a)}$ then we have \begin{align}V_h^k(s)-V_h^*(s)\leq\mathbb E^{s'\sim P_{s,a,h}}[V_{h+1}^k(s')-V_{h+1}^*(s')]+\left(H\land22H\sqrt{\frac{S\iota}{n_h^k(s,a)}}\right).\label{eq:localvdif} \end{align} If $\iota>n_h^k(s,a)$, then \Cref{eq:localvdif} also holds since $V_h^k(s)-V_h^*(s)\leq H$. Thus, we can recursively apply \Cref{eq:localvdif} and conclude \begin{align*}
        V_h^k(s)-V_h^*(s)\leq \mathbb E^{\pi^k}\left[\sum_{h'=h}^HH\land 22H\sqrt{\frac{S\iota}{n_{h'}^k(s_{h'},a_{h'})}}\Bigg | s_h=s\right]
    \end{align*}
\end{proof}

\begin{lemma}
    \label{thm:surplusub}
    Conditioned on success of \Cref{thm:concentration}, if $a=\pi_h^k(s)$ then \begin{align*}
        E_h^k(s,a)\leq \left(H\land 5\sqrt{\frac{\Var_h^*(s,a)\iota}{n_h^k(s,a)}}\right)+\mathbb E^{\pi^k}\left[\sum_{h'=h}^H3H^2\land \frac{1500SH^2\iota}{n_{h'}^k(s_{h'},a_{h'})}\Bigg|(s_h,a_h)=(s,a)\right].
    \end{align*}
\end{lemma}

\begin{proof}
    Similar to the proof of \Cref{thm:vtail}, \begin{align*}
        &E_h^k(s,a)=Q_h^k(s,a)-r_h(s,a)-\mathbb E^{s'\sim P_{s,a,h}}[V_{h+1}^k(s')] \\
        = & \hat r_h^k(s,a)+b_h^k(s,a)+\mathbb E^{s'\sim\hat P_{s,a,h}}V_{h+1}^k(s')-r_h(s,a)-\mathbb E^{s'\sim P_{s,a,h}}V_{h+1}^k(s') \\
        \leq &|\hat r_h^k(s,a)-r_h(s,a)|+|\mathbb E^{s'\sim\hat P_{s,a,h}}(V_{h+1}^k(s')-V_{h+1}^*(s'))-\mathbb E^{s'\sim P_{s,a,h}}(V_{h+1}^k(s')-V_{h+1}^*(s'))| \\
        & + |\mathbb E^{s'\sim \hat P_{s,a,h}}V_{h+1}^*(s')-\mathbb E^{s'\sim P_{s,a,h}}V_{h+1}^*(s')|+b_h^k(s,a).
    \end{align*}

    We will apply \Cref{thm:concentration,thm:p_concentrate,thm:bonus} to bound each term. So \begin{align*}
        &E_h^k(s,a)\leq \sqrt{\frac{2\mathbb V^{r'\sim R_{s,a,h}}[r']\iota}{n_h^k(s,a)}}+\frac{H\iota}{n_h^k(s,a)} \\
        &+\frac{1}{H}\mathbb E^{s'\sim P_{s,a,h}}[(V_{h+1}^k(s')-V_{h+1}^*(s'))^2]+\frac{SH\iota}{n_h^k(s,a)} \\
        &+\sqrt{\frac{2\mathbb V^{s'\sim P_{s,a,h}}[V_h^*(s')]\iota}{n_h^k(s,a)}}+\frac{H\iota}{n_h^k(s,a)} \\
        &+\frac{2}{H}\mathbb E^{s'\sim\hat P_{s,a,h}^k}[(V_{h+1}^k(s')-V_{h+1}^*(s'))^2]+2\sqrt{\frac{2\Var_h^*(s,a)\iota}{n_h^k(s,a)}}+\frac{20H\iota}{n_h^k(s,a)}.
    \end{align*}

    By \Cref{thm:p_concentrate} (with $V=(V_{h+1}^k-V_{h+1}^*)^2$) again,
    \begin{align}
        &\mathbb E^{s'\sim \hat P_{s,a,h}^k}[(V_{h+1}^k(s')-V_{h+1}^*(s'))^2]-\mathbb E^{s'\sim P_{s,a,h}}[(V_{h+1}^k(s')-V_{h+1}^*(s'))^2] \\
        \leq &\sqrt{\frac{2S\mathbb E^{s'\sim P_{s,a,h}}[(V_{h+1}^k(s')-V_{h+1}^k(s'))^4]\iota}{n_h^k(s,a)}}+\frac{SH^2\iota}{n_h^k(s,a)} \\
        \leq &\sqrt{\frac{2H^2S\mathbb E^{s'\sim P_{s,a,h}}[(V_{h+1}^k(s')-V_{h+1}^k(s'))^2]\iota}{n_h^k(s,a)}}+\frac{SH^2\iota}{n_h^k(s,a)} \\
        \leq &\mathbb E^{s'\sim P_{s,a,h}}[(V_{h+1}^k(s')-V_{h+1}^*(s'))^2]+\frac{2SH^2\iota}{n_h^k(s,a)}.
    \end{align}

    By \Cref{thm:vtail}, \begin{align*}&(V_{h+1}^k(s')-V_{h+1}^*(s'))^2\leq \mathbb E^{\pi^k}\left[\left(\sum_{h'=h+1}^HH\land 22H\sqrt{\frac{S\iota}{n_{h'}^k(s_{h'},a_{h'})}}\right)^2\Bigg|(s_h,a_h)=(s,a)\right] \\
    \leq &\mathbb E^{\pi^k}\left[\sum_{h'=h+1}^HH^3\land \frac{500SH^3\iota}{n_{h'}^k(s_{h'},a_{h'})}\Bigg|(s_h,a_h)=(s,a)\right].\end{align*}

    Hence, \begin{align*}
        &E_h^k(s,a)\leq(2+2\sqrt 2)\sqrt{\frac{\Var_h^*(s,a)\iota}{n_h^k(s,a)}}+\frac{3}{H}\mathbb E^{s'\sim P_{s,a,h}}[(V_{h+1}^k(s')-V_{h+1}^*(s'))^2]+\frac{27SH\iota}{n_h^k(s,a)} \\
        \leq & 5\sqrt{\frac{\Var_h^*(s,a)\iota}{n_h^k(s,a)}}+\mathbb E^{\pi^k}\left[\sum_{h'=h}^H3H^2\land \frac{1500SH^2\iota}{n_{h'}^k(s_{h'},a_{h'})}\Bigg|(s_h,a_h)=(s,a)\right].
    \end{align*}

    The extra ``$H\land$'' part is because $E_h^k(s,a)\leq H$ by definition.
\end{proof}

\subsection{Concentration of visitation count}

This lemma shows that the sum of visiting probabilities is bounded by $n_h^k$.

\begin{lemma}
    \label{thm:visitconcentration}
    With probability at least $1-\delta$,
    \[\sum_{k'=1}^k\mathbb E[\mathbf 1\{(s_{h}^{k'},a_{h}^{k'})=(s,a)\}|\mathcal F_{k'}]\leq3n_h^k(s,a)+\iota\] for all $s,a,h,k$, where $\mathcal F_{k}$ is the $\sigma$-field generated by the first $k-1$ episodes.
\end{lemma}

\begin{proof}
    This is a direct consequence of \Cref{lem:martingale}.
\end{proof}

The next lemma considers a weighted sum over visiting probabilities.

\begin{lemma}
    \label{thm:visitconcentrationV}
    With probability at least $1-2\delta$,
    \[\sum_{k'=1}^k\sum_{h'=1}^hw_h(s_{h'}^{k'},a_{h'}^{k'})\mathbb E[\mathbf 1\{(s_{h}^{k'},a_{h}^{k'})=(s,a)\}|\mathcal F_{k',h'}]\leq9\bar Wn_h^k(s,a)+4H\bar W\iota,\] for all $s,a,h,k$, where we recall the definition \[w_h(s,a)=\Var_h^*(s,a),\bar W=\min\{160H^2\log(4K(H+1)/\delta),\cVarmax\}\] and $\mathcal F_{k,h}$ is the $\sigma$-field generated by the first $(k-1)$ episodes and the first $h$ steps of the $k$-th episode.
\end{lemma}

\begin{proof}
    By \Cref{lem:martingale} and $w_h(s_{h'}^k,a_{h'}^k)\leq\bar W$, \[\sum_{k'=1}^k\sum_{h'=1}^hw_h(s_{h'}^{k'},a_{h'}^{k'})\mathbb E[\mathbf 1\{(s_{h},a_{h})=(s,a)\}|\mathcal F_{k',h'}]\leq 3\sum_{k'=1}^k\sum_{h'=1}^hw_h(s_{h'}^{k'},a_{h'}^{k'})\mathbf 1\{(s_{h},a_{h})=(s,a)\}+\bar W\iota\] for all $s,a,h,k$. Then, we will bound $\sum_{k'=1}^k\sum_{h'=1}^hw_h(s_{h'}^{k'},a_{h'}^{k'})\mathbf 1\{(s_{h},a_{h})=(s,a)\}$ in two different ways for each term in $\bar{W}$.

    First, we apply \Cref{lem:martingale} again with respect to only the sum over $k'$ with $(s_{h}^{k'},a_{h}^{k'})=(s,a)$. This shows \[\sum_{k'=1}^k\sum_{h'=1}^hw_h(s_{h'}^{k'},a_{h'}^{k'})\mathbf 1\{(s_{h}^{k'},a_{h}^{k'})=(s,a)\}\leq 3\cVarmax n_h^k(s,a)+H\cVarmax\iota.\]

    Second, by \Cref{thm:variancebound}, with probability $1-\delta$, \[\sum_{h'=1}^Hw_h(s_{h'}^{k'},a_{h'}^{k'})\mathbf 1\{(s_h^{k'},a_h^{k'})=(s,a)\}\leq 160H^2\log(4K(H+1)/\delta)\mathbf 1\{(s_h^{k'},a_h^{k'})=(s,a)\}\] for all $k'=1,2,\cdots,K$. Thus, \[\sum_{k'=1}^k\sum_{h'=1}^hw_h(s_{h'}^{k'},a_{h'}^{k'})\mathbf 1\{(s_{h}^{k'},a_{h}^{k'})=(s,a)\}\leq160H^2\log(4K(H+1)/\delta)n_h^k(s,a).\] Hence we conclude \[\sum_{k'=1}^k\sum_{h'=1}^hw_h(s_{h'}^{k'},a_{h'}^{k'})\mathbb E[\mathbf 1\{(s_{h}^{k'},a_{h}^{k'})=(s,a)\}|\mathcal F_{k',h'}]\leq 9\bar Wn_h^k(s,a)+4H\bar W\iota.\]
\end{proof}

\begin{lemma}
    \label{thm:integration_bare}
    Let $f$ be a non-increasing nonnegative function defined on $[1,+\infty)$ with $f(1)\leq M$. Conditioned on success event of \Cref{thm:visitconcentration}, we have \[\sum_{k=1}^Kf(n_h^k(s,a))\mathbb P[(s_{h}^{k},a_{h}^{k})=(s,a)|\mathcal F_{k}]\leq M(\iota+3)+3\int_1^{n_h^K(s,a)}f(x)\mathrm dx.\]
\end{lemma}

\begin{proof}
    Let $n_k'=\sum_{k'=1}^k\mathbb P[(s_{h}^{k'},a_{h}^{k'})=(s,a)|\mathcal F_{k'}]\leq3n_h^k(s,a)+\iota$ and \[K_0=\min\{k:n_k'\geq \iota+3\}.\] (If $n_K'<\iota+3$ then we define $K_0=K$.) Then, \begin{align*}
        & \sum_{k=1}^Kf(n_h^k(s,a))\mathbb P[(s_{h}^{k},a_{h}^{k})=(s,a)|\mathcal F_{k}] \\
        = & \sum_{k=1}^{K_0}f(n_h^k(s,a))\mathbb P[(s_{h}^{k},a_{h}^{k})=(s,a)|\mathcal F_{k}]+\sum_{k=K_0+1}^Kf(n_h^k(s,a))\mathbb P[(s_{h}^{k},a_{h}^{k})=(s,a)|\mathcal F_{k}] \\
        \leq & M\sum_{k=1}^{K_0}\mathbb P[(s_{h}^{k},a_{h}^{k})=(s,a)|\mathcal F_{k}]+\sum_{k=K_0+1}^Kf((n_k'-\iota)/3)(n_k'-n_{k-1}') \\
        \leq & Mn_{K_0}'+\sum_{k=K_0+1}^K\int _{n_{k-1}'}^{n_k'}f((x-\iota)/3)\mathrm dx \\
        \leq &M(\iota+3)+\int_{n_{K_0}'}^{n_K'}f((x-\iota)/3)\mathrm dx \leq M(\iota+3)+\int_{\iota+3}^{n_K'}f((x-\iota)/3)\mathrm dx \\
        = & M(\iota+3)+3\int_1^{(n_K'-\iota)/3}f(x)\mathrm dx \leq M(\iota+3)+3\int_1^{n_h^K(s,a)}f(x)\mathrm dx.
    \end{align*}
\end{proof}

\begin{lemma}
    \label{thm:integration_weighted}
    Let $f$ be a non-increasing nonnegative function defined on $[0,+\infty)$ with upper bound $M$. Conditioned on success event of \Cref{thm:visitconcentrationV}, we have \[\sum_{k=1}^Kf(n_h^k(s,a))\sum_{h'=1}^hw_h(s_{h'}^k,a_{h'}^k)\mathbb P[(s_{h}^{k},a_{h}^{k})=(s,a)|\mathcal F_{k,h'}]\leq M\bar W(4H\iota+9)+9\bar W\int_1^{n_h^K(s,a)}f(x)\mathrm dx.\]
\end{lemma}

The proof is similar to that of \Cref{thm:integration_bare}.

\subsection{Final calculations}

Our calculations are conditioned on success of \Cref{thm:concentration,thm:optimism,thm:visitconcentration,thm:visitconcentrationV}, and they happen simultaneously with probability at least $1-20\delta$.

We begin by analyzing the clipped surplus. By \Cref{thm:distributeclipping,thm:surplusub}, we have \begin{align*}&\bar E_h^k(s,a) \\
\leq &2\clip{H\land 5\sqrt{\frac{\Var_h^*(s,a)\iota}{n_h^k(s,a)}}}{\frac{\gapmin\Var_h^*(s,a)}{24(H^2\land \cVarmax)}} \\
 & +2\clip{\mathbb E^{\pi^k}\left[\sum_{h'=h}^H3H^2\land \frac{1500SH^2\iota}{n_{h'}^k(s_{h'},a_{h'})}\Bigg|(s_h,a_h)=(s,a)\right]}{\frac{\gapmin}{24H^2}} \\
= &2\sum_{s',a'}\mathbf 1\{(s',a')=(s,a)\}\clip{H\land 5\sqrt{\frac{\Var_h^*(s',a')\iota}{n_h^k(s',a')}}}{\frac{\gapmin\Var_h^*(s',a')}{24(H^2\land \cVarmax)}} \\
 & +2\clip{\sum_{s',a'}\sum_{h'=h}^H\mathbb E^{\pi^k}\left[\left(3H^2\land \frac{1500SH^2\iota}{n_{h'}^k(s',a')}\right)\mathbf 1\{(s_{h'},a_{h'})=(s',a')\}\Bigg| (s_h,a_h)=(s,a)\right]}{\frac{\gapmin}{24H^2}} \\
\leq &2\sum_{s',a'}\mathbf 1\{(s',a')=(s,a)\}f_h(s,a;n_h^k(s,a)) \\
 & +4\sum_{s',a'}\sum_{h'=h}^H\clip{\left(3H^2\land \frac{1500SH^2\iota}{n_{h'}^k(s,a)}\right)\mathbb E^{\pi^k}\left[\mathbf 1\{(s_{h'},a_{h'})=(s,a)\}| (s_h,a_h)=(s,a)\right]}{\frac{\gapmin}{48SAH^3}} \\
\leq &2\sum_{s',a'}\mathbf 1\{(s',a')=(s,a)\}f_h(s,a;n_h^k(s,a)) \\
 & +4\sum_{'s,a'}\sum_{h'=h}^H\mathbb E^{\pi^k}\left[\mathbf 1\{(s,a)=(s_{h'},a_{h'})\}| (s_h,a_h)=(s,a)\right]\clip{3H^2\land \frac{1500SH^2\iota}{n_{h'}^k(s,a)}}{\frac{\gapmin}{48SAH^3}} \\
 = & 2\sum_{s',a'}\mathbf 1\{(s',a')=(s,a)\}f_h(s,a;n_h^k(s,a)) \\
 & +4\sum_{s',a'}\sum_{h'=h}^H\mathbb P^{\pi^k}[(s',a')=(s_{h'},a_{h'})|(s_h,a_h)=(s,a)]g(n_{h'}^k(s',a')),\end{align*} where \begin{align*}
     f_h(s,a;x)=\clip{H\land 5\sqrt{\frac{\Var_h^*(s,a)\iota}{x}}}{\frac{\gapmin\Var_h^*(s,a)}{24(H^2\land \cVarmax)}}, g(x)=\clip{3H^2\land \frac{1500SH^2\iota}{x}}{\frac{\gapmin}{48SAH^3}}.
 \end{align*}

\subsubsection{Bounding regret}

By \Cref{thm:regretub}, we have \begin{align*}
    &\Regret(K)\leq \frac 32\sum_{k=1}^K\sum_{h=1}^H\mathbb E^{\pi^k}\left[\bar E_h^k(s_h,a_h)\right]=\frac32\sum_{k=1}^K\sum_{h=1}^H\mathbb E[\bar E_h^k(s_h^k,a_h^k)|\mathcal F_{k}] \\
    \leq & 3\sum_{k=1}^K\sum_{h=1}^K\sum_{s,a}\mathbb E\left[\mathbf 1\{(s,a)=(s_h^k,a_h^k)\}f_h(s,a;n_h^k(s,a))|\mathcal F_{k}\right] \\
    & + 6\sum_{k=1}^K\sum_{h=1}^K\sum_{s,a}\sum_{h'=h}^H\mathbb E\left[\mathbb P^{\pi^k}\left[(s,a)=(s_{h'},a_{h'})| (s_h,a_h)=(s_h^k,a_h^k)\right]g(n_{h'}^k(s,a))|\mathcal F_{k}\right] \\
    = & 3\sum_{s,a}\sum_{h=1}^K\sum_{k=1}^Kf_h(s,a;n_h^k(s,a))\mathbb P[(s,a)=(s_h^k,a_h^k)|\mathcal F_{k}] \\
    & + 6\sum_{h=1}^H\sum_{s,a}\sum_{h'=h}^H\sum_{k=1}^Kg(n_{h'}^k(s,a))\mathbb P\left[(s,a)=(s_{h'}^k,a_{h'}^k)|\mathcal F_{k}\right].
\end{align*}

We will use \Cref{thm:integration_bare} to bound the two sums. 

For the first sum, we have then \begin{align*}
    & \sum_{k=1}^Kf_h(s,a;n_h^k(s,a))\mathbb P[(s,a)=(s_h^k,a_h^k)] \leq H(\iota+3)+3\int_1^{n_h^K(s,a)}f_h(s,a;n_h^k(s,a))\mathrm dx.
\end{align*}

When $\Delta_h(s,a)>0$, we bound the integration by \begin{align*}
    & \int_1^{n_h^K(s,a)}f_h(s,a;n_h^k(s,a))\mathrm dx=\int_1^{n_h^K(s,a)}\clip{5\sqrt{\frac{\Var_h^*(s,a)\iota}{x}}}{\frac{\gapmin\Var_h^*(s,a)}{24(H^2\land\cVarmax)}}\mathrm dx \\
    \leq & \int_0^{n_h^K(s,a)}5\sqrt{\frac{\Var_h^*(s,a)\iota}{x}}\mathrm dx=10\sqrt{\Var_h^*(s,a)n_h^K(s,a)\iota}.
\end{align*}

When $\Delta_h(s,a)=0$, we bound by \begin{align*}
    &\int_1^{n_h^K(s,a)}f_h(s,a;n_h^k(s,a))\mathrm dx=\int_1^{n_h^K(s,a)}\clip{5\sqrt{\frac{\Var_h^*(s,a)\iota}{x}}}{\frac{\gapmin\Var_h^*(s,a)}{24(H^2\land\cVarmax)}}\mathrm dx \\
    \leq &\int_0^{+\infty}\clip{5\sqrt{\frac{\Var_h^*(s,a)\iota}{x}}}{\frac{\gapmin\Var_h^*(s,a)}{24(H^2\land\cVarmax)}}\mathrm dx \\
    = & \int_0^{(120(H^2\land\cVarmax))^2\iota/\gapmin^2\Var_h^*(s,a)}5\sqrt{\frac{\Var_h^*(s,a)\iota}{x}}\mathrm dx \\
    = & \frac{1200(H^2\land\cVarmax)\iota}{\gapmin}.
\end{align*}

For the second sum, we bound similarly that \begin{align*}
    & \sum_{h'=h}^Hg(n_{h'}^k(s,a))\mathbb P\left[(s,a)=(s_{h'}^k,a_{h'}^k)\right]\\
    \leq & 3H^2(\iota+3)+3\int_1^{n_h^K(s,a)}\clip{\frac{1500SH^2\iota}{x}}{\frac{\gapmin}{48SAH^3}}\mathrm dx \\
    \leq & 3H^2(\iota+3)+3\int_1^{72000S^2AH^5\iota/\gapmin}\frac{1500SH^2\iota}{x}\mathrm dx \\
    = & 3H^2(\iota+3)+4500SH^2\iota\log(72000S^2AH^5\iota/\gapmin).
\end{align*}

Thus, \begin{align}
\Regret(K)\leq &3\sum_{(s,a,h)\in\Zsub}(4H\iota+30\sqrt{\Var_h^*(s,a)n_h^K(s,a)\iota}) + 3\sum_{(s,a,h)\in Zopt}\left(4H\iota+\frac{3600(H^2\land\cVarmax)\iota}{\gapmin}\right) \notag \\
& + 6SAH^2(12H^2+4500SH^2\iota\log(72000S^2AH^5\iota/\gapmin))\notag\\
 \leq& 90\sum_{(s,a,h)\in\Zsub}\sqrt{\Var_h^*(s,a)n_h^K(s,a)\iota} +10800\sum_{(s,a,h)\in Zopt}\frac{(H^2\land\cVarmax)\iota}{\gapmin} \notag\\
 &+27000S^2AH^4\iota\log(72000S^2AH^5\iota/\gapmin)+96SAH^5\iota.\label{eq:regret}
 \end{align}

\subsubsection{Lower-bounding visitation count}

Recall the lower bound in \Cref{thm:gaplbsurplus}. We begin by taking a weighted sum over all states visited during the algorithm:
\begin{align*}
    &\sum_{k=1}^K\sum_{h=1}^Hw_h(s_h^k,a_h^k)\Delta_h(s_h^k,a_h^k)\leq\frac32\sum_{k=1}^K\sum_{h=1}^Hw_h(s_h^k,a_h^k)\sum_{h'=h}^H\mathbb E[\bar E_{h'}^k(s_{h'}^k,a_{h'}^k)|\mathcal F_{k,h}] \\
    \leq & 3 \sum_{k=1}^K\sum_{h=1}^Hw_h(s_h^k,a_h^k)\sum_{h'=h}^H\sum_{s,a}\mathbb E\left[\mathbf 1\{(s,a)=(s_{h'}^k,a_{h'}^k)\}f_h(s,a;n_{h'}^k(s,a))|\mathcal F_{k,h}\right] \\
    & + 6\sum_{k=1}^K\sum_{h=1}^Kw_h(s_h^k,a_h^k)\sum_{h'=h}^H\sum_{s,a}\sum_{h^*=h'}^H\mathbb E\left[\mathbb P^{\pi^k}\left[(s,a)=(s_{h^*},a_{h^*})| (s_{h'},a_{h'})=(s_{h'}^k,a_{h'}^k)\right]g(n_{h^*}^k(s,a))|\mathcal F_{k,h}\right] \\
    = & 3 \sum_{s,a}\sum_{h'=1}^H\sum_{k=1}^Kf_h(s,a;n_{h'}^k(s,a))\sum_{h=1}^{h'}w_h(s_h^k,a_h^k)\mathbb P\left[(s,a)=(s_{h'}^k,a_{h'}^k)|\mathcal F_{k,h}\right] \\
    & + 6\sum_{s,a}\sum_{h^*=1}^H\sum_{k=1}^Kg(n_{h^*}^k(s,a))\sum_{h=1}^{h^*}\sum_{h'=h}^{h^*}w_h(s_h^k,a_h^k)\mathbb P[(s,a)=(s_{h^*}^k,a_{h^*}^k)| \mathcal F_{k,h}] \\
    = & 3 \sum_{s,a}\sum_{h'=1}^H\sum_{k=1}^Kf_h(s,a;n_{h'}^k(s,a))\sum_{h=1}^{h'}w_h(s_h^k,a_h^k)\mathbb P\left[(s,a)=(s_{h'}^k,a_{h'}^k)|\mathcal F_{k,h}\right] \\
    & + 6H\sum_{s,a}\sum_{h^*=1}^H\sum_{k=1}^Kg(n_{h^*}^k(s,a))\sum_{h=1}^{h^*}w_h(s_h^k,a_h^k)\mathbb P[(s,a)=(s_{h^*}^k,a_{h^*}^k)| \mathcal F_{k,h}].
\end{align*}

Take $w_h(s,a)=\Var_h^*(s,a)$, it follows from \Cref{thm:integration_weighted} that
\begin{align}
    & \sum_{s,a}\sum_{h=1}^Hw_h(s,a)\Delta_h(s,a)n_h^k(s,a)=\sum_{k=1}^K\sum_{h=1}^Hw_h(s_h^k,a_h^k)\Delta_h(s_h^k,a_h^k) \notag \\
    \leq &3\sum_{s,a}\sum_{h'=1}^H\bar W\left(3H^2(4H\iota+9)+9\int_1^{n_h^K(s,a)}f_{h'}(s,a;x)\mathrm dx\right) \notag \\
    & +6H\sum_{s,a}\sum_{h^*=1}^H\bar W\left(3H^2(4H\iota+9)+9\int_1^{n_h^K(s,a)}f_{h^*}(s,a;x)\mathrm dx\right) \notag \\
    \leq &351\bar WSAH^5\iota+270\bar W\sum_{(s,a,h)\in\Zsub}\sqrt{\Var_h^*(s,a)n_h^K(s,a)\iota} \notag \\
    & +\frac{32400|\Zopt|\bar W(H^2\land\cVarmax)\iota}{\gapmin}+81000S^2AH^4\iota\bar W\log(72000S^2AH^5\iota\gapmin). \label{eq:regretubconstructed}
\end{align}

For notational simplicity, we denote
\[R_0=\sum_{(s,a,h)\in\Zsub}\sqrt{\Var_h^*(s,a)\Delta_h(s,a)\iota}.\]
By \Cref{eq:regretubconstructed} and Cauchy-Schwarz inequality, \begin{align*}
    & \bar W\left(351SAH^5\iota+270R_0 
     +\frac{32400|\Zopt|(H^2\land\cVarmax)\iota}{\gapmin}+81000S^2AH^4\iota\log(72000S^2AH^5\iota\gapmin)\right) \\
     & \cdot\left(\sum_{(s,a,h)\in\Zsub}\frac{\iota}{\Delta_h(s,a)}\right)\geq R_0^2.
\end{align*}

It follows by solving the quadratic equation that \[R_0\leq\sum_{(s,a,h)\in\Zsub}\frac{540\bar W\iota}{\Delta_h(s,a)}+2SAH^5\iota 
     +\frac{120|\Zopt|(H^2\land\cVarmax)\iota}{\gapmin}+300S^2AH^4\iota\log(72000S^2AH^5\iota\gapmin).\]

From \Cref{eq:regret}, \begin{align*}
    & \Regret(K)\leq 90R_0+96SAH^4\iota+10800\sum_{(s,a,h)\in Zopt}\frac{(H^2\land\cVarmax)\iota}{\gapmin}+27000S^2AH^3\iota\log(72000S^2AH^5\iota/\gapmin) \\
    \leq & \sum_{(s,a,h)\in\Zsub}\frac{48600\bar W\iota}{\Delta_h(s,a)}+\frac{21600|\Zopt|(H^2\land\cVarmax)\iota}{\gapmin}+270000S^2AH^4\iota\log(10SAH\iota/\gapmin) + 276SAH^5\iota,
\end{align*} with probability at least $1-20\delta$, as we have claimed in the main text.

Thus we have proved the following main theorem.

\begin{theorem} [Formal statement of \Cref{thm:upper_bound_informal}] \label{thm:upper_bound}
    Suppose we run MVP algorithm with universal constants $c_1=c_2=2,c_3=10$. For any MDP instance $\mathcal M$ satisfying \Cref{asp:bounded_total_reward} and any confidence parameter $\delta>0$, any episode number $K\geq 1$, with probability at least $1-20\delta$, \begin{align*}
    \Regret(K)&\lesssim\sum_{(s,a,h)\in\Zsub}\frac{(H^2\log(HK/\delta)\land\cVarmax)\log(SAHK/\delta)}{\Delta_h(s,a)} \\
    &+\frac{|\Zopt|(H^2\land \cVarmax)\log(SAHK/\delta)}{\gapmin} \\ 
    &+S^2AH^4\log(SAHK/\delta)\log(SAH\gapmin^{-1}\log(SAHK/\delta)) \\
    &+SAH^5\log(SAHK/\delta).
    \end{align*}
\end{theorem}

\section{Regret Lower Bound}\label{app:lb}

\begin{theorem} [Formal statement of \Cref{thm:lb}] \label{thm:lb_formal}
For a given configuration of $S,A,H$, target conditional variance $L\in [1, H^2]$, as well as a set of suboptimality gaps $\boldsymbol{\Delta} = \{ \Delta_1, \Delta_2, \ldots, \Delta_{S A H}\}$, we make the following mild assumptions:

\begin{itemize}

\item Let $\cI = \{ i \mid \Delta_i = 0\}$.
Assume that $\abs{\cI} \ge S H$, i.e., the suboptimality gaps are realizable.

\item Assume that $\Delta_i < \sqrt{L}$ for all $1 \le i \le S A H$.

\end{itemize}

For any algorithm $\boldsymbol{\pi}$, there exists an MDP instance $\cM^{\boldsymbol{\pi}}$ satisfying:

\begin{itemize}

\item It has $\abs{\bar{\cS}} = S + 2$ states and $A$ actions.

\item There exists $\cS \subset \bar{\cS}$ such that $\abs{\cS} = S$, and a bijection $\sigma$ between $[H] \times \cS \times \cA$ and $[S A H]$, satisfying $\Delta_h (s, a) = \frac{1}{4} \Delta_{\sigma (h, s, a)}$ for any $(h, s, a) \in [H] \times \cS \times \cA$.

\item $\cVarmax = \Theta (L)$, while $\Varmax \le O (1)$.

\end{itemize}

such that
\begin{align*}
\lim_{K \to \infty} \frac{\bbE^{\boldsymbol{\pi}} [\Regret (\cM^{\boldsymbol{\pi}}, K)]}{\log K}
 \ge \Omega \left( \sum_{i : \Delta_i>0} \frac{L}{\Delta_i}\right).
\end{align*}
\end{theorem}

\begin{proof}
First consider multi-armed bandit lower bound given a set of gaps $\boldsymbol{\Delta} = \{\Delta_1, \Delta_2, \ldots, \Delta_A\}$ and a target variance $L$.
WLOG, assume $\Delta_i \le \Delta_{i + 1}$.
Construct Bernoulli outcomes for each action $a_i$: w.p. $p_i = \frac{1}{2} - \frac{\Delta (a_i)}{4 \sqrt{L}} \in [\frac{1}{4}, \frac{1}{2}]$, get reward $\sqrt{L}$; w.p. $1 - p_i$, get reward $0$.
Then $Q (a_i) = p_i = \frac{1}{2} - \frac{\Delta (a_i)}{4 \sqrt{L}}$, and $Q (a_1) - Q (a_i) = \frac{\Delta (a_i)}{4 \sqrt{L}}$.
Then $\Var (a_i) = p_i (1 - p_i) L = \Theta (L)$.
We invoke standard lower bound \citep{lai1985asymptotically} with reward outcomes in $[0, 1]$.
We first scale the rewards in our example by $\frac{1}{\sqrt{L}}$.
For any algorithm $\boldsymbol{\pi}$, there exists a permutation on the gaps (into $\frac{1}{\sqrt{L}} \boldsymbol{\Delta}^{\boldsymbol{\pi}}$), such that
\begin{align*}
    \lim_{K \to \infty} \frac{\bbE[\Regret(\frac{1}{\sqrt{L}} \boldsymbol{\Delta}^{\boldsymbol{\pi}}, K)]}{\log K}
    \ge \sum_i \frac{Q (a_1) - Q (a_i)}{\kl (p_i, \frac{1}{2})}
    \mygei \sum_{i : \Delta_i > 0} \frac{\Var (a_i)}{Q (a_1) - Q (a_i)}
    \ge \Omega \left( \sum_{i : \Delta_i > 0} \frac{1}{\Delta_i / \sqrt{L}} \right),
\end{align*}
where $\kl (p, q) = p \log \frac{p}{q} + (1 - p) \log \frac{1 - p}{1 - q}$;
(i) is by $\frac{(\frac{1}{2} - x)^2}{x (1 - x)} \ge x \log (2 x) + (1 - x) \log (2 - 2 x)$ for $x \in [0, 1]$ (we take $x = p_i$). To see this, we substitute $t=1-2x\in[-1,1]$, then \[\frac{(\frac12-x)^2}{x(1-x)}=\frac{t^2}{(1-t)(1+t)}\geq t^2\] and \[x\log(2x)+(1-x)\log(2-2x)=\frac{1-t}2\log(1-t)+\frac{1+t}{2}\log(1+t)\leq \frac{-t(1-t)}{2}+\frac{t(1+t)}{2}=t^2.\]
Scaling back, we have
\begin{align*}
    \lim_{K \to \infty} \frac{\bbE[\Regret(\boldsymbol{\Delta}^{\boldsymbol{\pi}}, K)]}{\log K}
    = \lim_{K \to \infty} \frac{\sqrt{L} \bbE[\Regret(\frac{1}{\sqrt{L}}\boldsymbol{\Delta}^{\boldsymbol{\pi}}, K)]}{\log K}
    \ge \Omega \left( \sum_{i: \Delta_i > 0} \frac{L}{\Delta (a_i)} \right).
\end{align*}

Then, we construct the MDP as:
\begin{itemize}
    \item \textbf{States:} in total $S + 2$ states.
    $s_0$ as a main state, $s_1, s_2, \ldots, s_S$ as bandit states, $s_{-1}$ as a terminal state.

    \item \textbf{Transition:} $s_0$ does not require decision-making: $P_{s_0, a, h} (s_0) = 1 - \frac{1}{L H}$, $P_{s_0, a, h} (s_i) = \frac{1}{L S H}$ for $1 \le i \le S$.
    $s_i$ is a bandit problem, and directly transits into $s_{-1}$: $P_{s_i, a, h} (s_{-1}) = 1$ for $1 \le i \le S$.
    $s_{-1}$ is self-absorbing: $P_{s_{-1}, a, h} (s_{-1}) = 1$.

    \item \textbf{Rewards:} for $s_0$ and $s_{-1}$, all rewards are $0$.
    Rewards for $(s_i, a, h)$ are decided by the construction below.
\end{itemize}
Assign $\boldsymbol{\Delta}$ into $H \times S$ groups, each with exactly $A$ items: $\{\boldsymbol{\Delta}_{h, s_i}\}_{(h, i) \in [H] \times [S]}$ and from the assumption we can guarantee at least one $0$ gap in each group.
We have $d_h (s_i) = \frac{1}{L S H} (1 - \frac{1}{L H})^{h - 1} \in [\frac{1}{\ee L S H}, \frac{1}{L S H}]$ for $1 \le i \le S$.
For each $(h, i) \times [H] \times [S]$, from \Cref{thm:bennett}, with probability at least $1 - \frac{1}{2 H S}$,
\begin{align*}
    \abs{d_h (s_i) - \frac{N_h^K (s_i)}{K}} &\le \sqrt{\frac{2 d_h (s_i) (1 - d_h (s_i)) \log (4 S H)}{K}} + \frac{\log (4 S H)}{K} \\
    \Rightarrow K d_h (s_i) - N_h^K (s_i) &\le \sqrt{\frac{2 K}{L S H} \log (4 S H)} + \log (4 S H).
\end{align*}
When $K \ge 2 \ee^2 (1 + \sqrt{1 + \ee})^2 L S H \log (4 S H)$, we have RHS $\le \frac{K}{2 \ee L S H}$, so we have $N_h^k (s_i) \ge K d_h (s_i) - \frac{K}{2 \ee L S H} \ge \frac{K}{2 \ee L S H}$.
Denote the event $\cE = \{N_h^K (s_i) \ge \frac{K}{2 \ee L S H} \mid (h, i) \in [H] \times [S]\}$, then $\bbP [\cE] \ge \frac{1}{2}$.

Since we set independent random instances for each $(h, s_i)$, we have that
\begin{align*}
    \lim_{K \to \infty} \frac{\bbE [\Regret (\boldsymbol{\Delta}_{h, s_i}^{\boldsymbol{\pi}}, \boldsymbol{\pi}, K)]}{\log K}
    & \ge \lim_{K \to \infty} \frac{\bbP [\cE] \bbE [\Regret (\boldsymbol{\Delta}_{h, s_i}^{\boldsymbol{\pi}}, \frac{K}{2 \ee L S H})] + (1 - \bbP [\cE]) \cdot 0}{\log K} \\
    & \mygei \lim_{K \to \infty} \frac{\bbE [\Regret (\boldsymbol{\Delta}_{h, s_i}^{\boldsymbol{\pi}}, \frac{K}{2 \ee L S H})]}{4 \log (\frac{K}{2 \ee L S H})} \\
    & \ge  \Omega \left( \sum_{a: \Delta_{h, s_i} (a) > 0} \frac{L}{\Delta_{h, s_i} (a)} \right),
\end{align*}
where (i) is by $\bbP[\cE] \ge \frac{1}{2}$ and taking $K \ge (2 \ee L S H)^2$.
So
\begin{align*}
    \lim_{K \to \infty} \frac{\bbE [\Regret (\cM^{\boldsymbol{\pi}}, \boldsymbol{\pi}, K)]}{\log K}
    &= \lim_{K \to \infty} \sum_{h, i} \frac{\bbE [\Regret (\boldsymbol{\Delta}_{h, s_i}^{\boldsymbol{\pi}}, \boldsymbol{\pi}, K)]}{\log K} \\
    &= \sum_{h, i} \lim_{K \to \infty} \frac{\bbE [\Regret (\boldsymbol{\Delta}_{h, s_i}^{\boldsymbol{\pi}}, \boldsymbol{\pi}, K)]}{\log K} \\
    &\ge \Omega \left( \sum_{(h, i, a): \Delta_{h, s_i} (a) > 0} \frac{L}{\Delta_{h, s_i} (a)} \right) \\
    &= \Omega \left( \sum_{i: \Delta_i > 0} \frac{L}{\Delta_i} \right).
\end{align*}

We have $\Var_h^* (s_0) = \Theta ((1 - \frac{1}{L H}) \frac{1}{L H} \cdot L) = \Theta (\frac{1}{H})$, $\Var_h^* (s_{-1}) = 0$, and $\Var_h^* (s_i) = \Theta (L)$.
It is easy to verify that $\cVarmax$ is taken at states $(h, s_i)$, so
\begin{align*}
    \cVarmax = \max_{h, i} \left\{ \Var_h^* (s_i) + \sum_{t = 1}^{h - 1} \Var_t^* (s_0) \right\} = \Theta (L).
\end{align*}
However,
\begin{align*}
    \Varmax \le \sum_{h = 1}^H \left(d_h (s_0) \Var_h^* (s_0) + \sum_{i = 1}^S d_h (s_i) \Var_h^* (s_i) \right) \le O (1),
\end{align*}
showcasing the separation between $\cVarmax$ and $\Varmax$.
\end{proof}

\end{document}